\documentclass{article}

\PassOptionsToPackage{numbers, sort&compress}{natbib}

 \usepackage[preprint]{neurips_2026}

\usepackage[utf8]{inputenc} 
\usepackage[T1]{fontenc}    
\usepackage{hyperref}       
\usepackage{url}            
\usepackage{booktabs}       
\usepackage{amsfonts}       
\usepackage{nicefrac}       
\usepackage{microtype}      
\usepackage[dvipsnames]{xcolor}
\usepackage{xfrac}          
\usepackage{mathtools}     
\usepackage{amsmath}\allowdisplaybreaks
\usepackage{bm}
\usepackage{amssymb}
\numberwithin{equation}{section}
\usepackage[colorinlistoftodos]{todonotes}
\usepackage{comment}
\usepackage{subcaption}
\usepackage{wrapfig}
\usepackage{tcolorbox}
\usepackage{enumitem}
\usepackage{pifont}
\usepackage{subcaption}
\usepackage{placeins}

\usepackage{amsthm}
\usepackage{float}
\usepackage{thmtools, thm-restate}
\usepackage{hyperref}
\usepackage[nameinlink,noabbrev]{cleveref}

\newtheorem{proposition}{Proposition}[section]
\newtheorem{definition}{Definition}[section]

\crefname{theorem}{theorem}{theorems}
\Crefname{theorem}{Theorem}{Theorems}
\crefname{proposition}{proposition}{propositions}
\Crefname{proposition}{Proposition}{Propositions}
\crefname{definition}{Definition}{Definitions}
\Crefname{definition}{Definition}{Definitions}
\crefname{assumption}{Assumption}{Assumptions}
\Crefname{assumption}{Assumption}{Assumptions}

\DeclarePairedDelimiter{\abs}{\lvert}{\rvert}   
\DeclarePairedDelimiterX{\ip}[2]{\langle}{\rangle}{#1,\,#2}

\theoremstyle{remark}

\usepackage{pifont}
\newcommand{\norm}[1]{\left\lVert #1 \right\rVert}

\hypersetup{
	colorlinks=true,
	linkcolor=myred,
	citecolor=myred,
	urlcolor=mygray,
}

\definecolor{myblue}{RGB}{0,0,150}
\definecolor{myred}{RGB}{120,50,50}
\definecolor{mygray}{RGB}{90,90,110}

\title{Beyond Uniform Local Isometry and Topology: FactoMap for Disentangled Representations}

\author{%
  Sohini Gupta$^{1,2}$, Michael Gao$^{1}$, and Bahareh Tolooshams$^{1,2,3}$\\
  $^1$University of Alberta, $^2$Alberta Machine Intelligence Institute (Amii), $^3$Canada CIFAR AI Chair\\
}

\newcommand{\R}{\mathbb{R}} 

\newcommand{\I}{{\bm I}}

\newcommand{\G}{{\bm G}}
\newcommand{\J}{{\bm J}}

\newcommand{\w}{{\bm w}}
\newcommand{\z}{{\bm z}}
\newcommand{\x}{{\bm x}}
\newcommand{\cvec}{{\bm c}}

\DeclareMathOperator*{\argmin}{arg\,min}

\begin{document}
\maketitle
%

%
\begin{abstract}
Many disentanglement methods represent generative factors using Euclidean product coordinates, although the underlying factor spaces may wrap, collapse, or have position-dependent geometry. We introduce \emph{factor-space structure}, which combines factor domains, generator-induced identifications, and position-dependent scales to distinguish topologically equivalent spaces with different factor geometries. We show that statistical independence does not imply uniform product geometry: independent hue and scale can affect observations at different, scale-dependent rates, producing anisotropy that no fixed rescaling can remove. We propose the \emph{Factor-Space Topographic Map} (FactoMap), which learns interpretable prototypes indexed by a structured factor-space lattice. Topographic learning transfers the lattice's periodicity, collapses boundaries, and non-uniform extent to the representation. Experiments show that matching this structure better preserves factor continuity and improves disentanglement of the underlying factors.
\end{abstract}
%

%
\section{Introduction}\label{sec:intro}
\vspace{-2mm}
Disentangled representation learning seeks to recover the factors that generate
a dataset, so that changing one factor of the world changes one coordinate of the representation while leaving the others fixed~\cite{bengio2013representation,higgins2018disentangle,desjardins2012disentangling}. Such representations are appealing because their coordinates can support interpretation, generalization, fairness, and control~\cite{chen2016infogan,ma2019learning,pmlr-v119-locatello20a,locatello2019fairness,roth2023disentanglement,hsu2023disentanglement}. Disentangling superposed features can also reveal shared representational content across neural systems that may be obscured by comparisons at the level of individual units~\cite{longon2025superposition}. However, the factors underlying an arbitrary generative process cannot be recovered from unlabelled observations without additional assumptions~\cite{locatello2019challenging}.

Existing approaches impose inductive biases through statistics or geometry. Statistical methods encourage factorized representations~\cite{higgins2017betavae,pmlr-v80-kim18b,chen2018isolating,kumar2018variational}, while geometric ones constrain how distances vary between factors and observations~\cite{gropp2020isometricautoencoders,huh2023isometric,song2023flow}. A prominent geometric assumption is local isometry: equal-sized changes in different factor directions should produce equal-sized changes in the observation, up to a fixed global choice of units. This enables identifiability under suitable conditions~\cite{horan2021when,nelson2026statistical}, and has motivated methods that favour flat or distance-preserving representations~\cite{lee2022regularized}.

Local isometry, however, is not generally a property of the data-generating process. Different factors act through different mechanisms, and their effects can depend on the current factor configuration. For an object with independently varying hue and scale, changing hue alters object area, whereas changing scale alters mainly its boundary. Because area and perimeter grow at different rates, their relative effects vary as the object grows: the mismatch is position-dependent and cannot be removed by a fixed rescaling of the units. Statistical independence therefore does not imply geometric separability. Factors may be sampled independently while remaining geometrically coupled through the generator.

Local geometry cannot determine whether a factor terminates or closes: an interval and a circle are locally indistinguishable, although a Euclidean coordinate must cut a periodic factor. Non-Euclidean topologies address this closure problem~\cite{davidson2018hyperspherical,moor2020topological,pmlr-v162-tonnaer22a,van2026constructing}, but topology~\cite{lee2000introduction} alone remains insufficient. A square, a planar disk, and a finite conical surface are homeomorphic~\cite{lee2000introduction}, yet support different factor organizations. A circular factor may retain constant extent, shrink, or collapse at one point. These spaces therefore share a topology while differing in their factor identifications and local scale-structure; a representation encoding only dimension and topology cannot express these distinctions.

\begin{figure}
    \centering 
    \includegraphics[width=\linewidth]{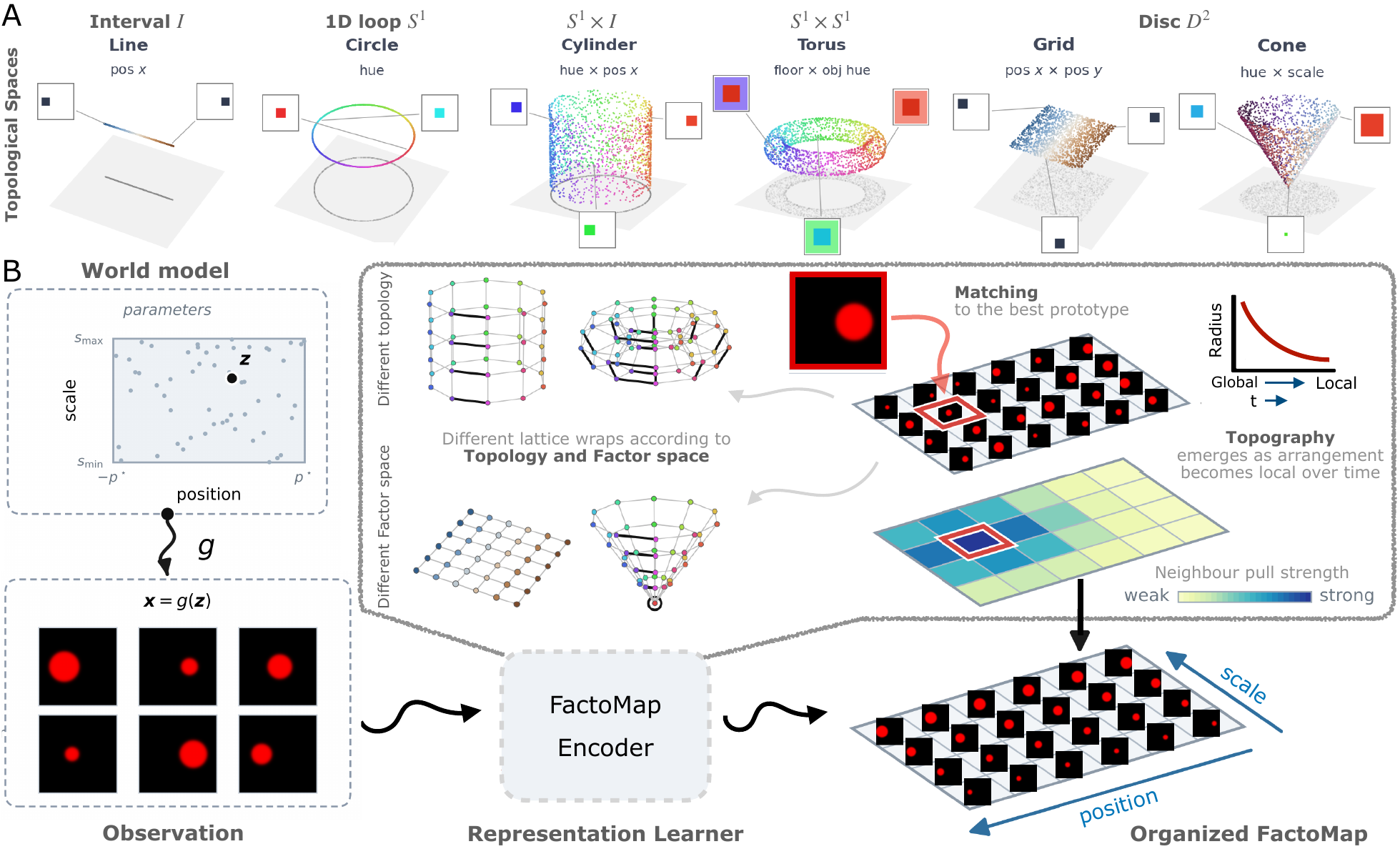}
    \caption{\textbf{FactoMaps: Factor-Space + Topology + Topography.} \textbf{A.} Visuals of different topological spaces; the grid and cone are homeomorphic but differ in factor-space geometry. \textbf{B.} {\it Left:} A generator maps two factors, position and scale, to observations. {\it Right:} FactoMap learns prototypes on a lattice encoding the factor topology and geometry, organizing its axes by position and scale.}
    \vspace{-6.2mm}
    \label{fig:factomap}
\end{figure}
We introduce \emph{factor-space structure} and the \emph{Factor-Space Topographic Map} (FactoMap,~\Cref{fig:factomap}), a structured prototype learner designed to disentangle data by aligning lattice coordinates with the underlying factors, even when they do not admit uniform Euclidean geometry. The formalism combines topology with generator-induced identifications and position-dependent scales; FactoMap realizes this structure as a lattice of interpretable prototypes, learned through topographic cooperation~\cite{kohonen1990self,welling2002learning}. We study the two-factor case and leave more factors to future work.
\begin{itemize}[leftmargin=4mm]
\setlength\itemsep{0.0em}
    \item \textit{Factor-space formalism.} We separate factor domains, generator-induced identifications, and local scales per factor, distinguishing homeomorphic spaces with different organizations or geometries.
    \item \textit{Failure of uniform local isometry.} We show that statistically independent hue and scale are geometrically coupled, producing position-dependent anisotropy that fixed rescaling cannot remove.
    \item \textit{Factor-space prototype disentanglement.} We propose FactoMap, whose supplied lattice distance encodes periodic closure, collapsed fibres, and non-uniform factor extent. Without observing factor labels, its topographic learning organizes the prototypes to vary smoothly along factor-space paths, making the lattice coordinates a structured representation aligned with the underlying factors.
    \item \textit{Failure to capture the topology of underlying factors.} We show that $\beta$-VAE~\cite{higgins2017betavae} fails to recover cyclic factor topology. Experiments reveal lower disentanglement scores of $\beta$-VAE for factor pairs involving orientation, demonstrating that aggregate scores can obscure factor-specific topological failures. FactoMap can successfully capture such topological structure and disentangle the factors.
\end{itemize}
%
%
\vspace{-2mm}
\section{Factor-Space Topographic Map (FactoMap)}\label{sec:methods}
\vspace{-2mm}
\textbf{Factor-space structure.}\quad The topology of the individual factor domains determines neither the identifications induced by the generator nor the observation-space sensitivity to each factor. We therefore describe factor-space structure using three ingredients $\mathfrak F_g=(\mathcal Z,\approx_g,\cvec_g)$ (\Cref{def:facspacestruc}). First, a \emph{domain} per factor: each $Z_i$ is 1D topological space (interval $I$ or a circle $S^1$), so that $\mathcal Z=\prod_{i=1}^{p}Z_i$ records which factors terminate and which are periodic. Second, the \emph{identifications} the generator $g:\mathcal Z\to\R^m$ induces, $\z\approx_g\z'$ whenever $g(\z)=g(\z')$, which record where distinct factor configurations produce the same data. Third, whenever $g$ is differentiable, a \emph{local scale} per factor,
\begin{equation}
    c_i(\z):=\|\tfrac{\partial g}{\partial z_i}(\z)\|_2,
    \qquad\text{so that}\qquad
    \norm{g(\z+\delta\bm e_i)-g(\z)}_2=c_i(\z)\abs{\delta}+o(\abs{\delta}),
    \label{eq:local-scale}
\end{equation}
with increments along a circular factor taken modulo its period. Thus, this local scale measures the observation-space displacement per unit change in factor $i$ at $\z$. 

The domains and identifications determine the quotient topology of the realized factor space $\mathcal F_g=\mathcal Z/\!\approx_g$, while the scale functions specify the diagonal local geometry. Consequently, factor spaces can be homeomorphic while having different geometries. Identification must be specified separately because they cannot be inferred from local scales: a generator may identify distinct configurations without its local scale vanishing, while a local scale may vanish without producing a global identification (\Cref{app:identifications}). Such identifications also permit non-product factor spaces. For example, collapsing a circular fibre at one endpoint convert a cylinder into a cone (\Cref{fig:factomap}A).

\textbf{Beyond uniform Euclidean factor geometry.}\quad Several geometric approaches to disentanglement ask the representation to be locally isometric to a Euclidean factor domain~\cite{horan2021when,gropp2020isometricautoencoders,lee2022regularized,huh2023isometric,nelson2026statistical}. In the language above, this imposes strong conditions on the generator: every $c_i(\z)$ must be constant in $\z$, equal across factors up to a single global scale, and the factor directions must be orthogonal (\Cref{app:isometry}). We show that even simple world models can violate these conditions.

\begin{wrapfigure}[11]{r}{0.33\linewidth}
  \vspace{-5mm}
  \centering
  \includegraphics[width=\linewidth]{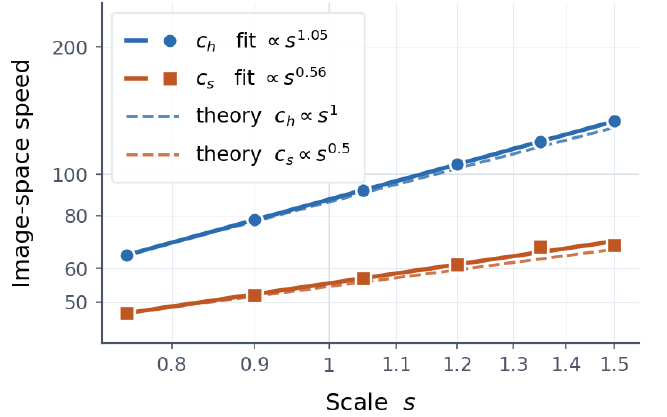}
  \vspace{-19pt}
  \caption{The generator induces position-dependent anisotropy.}
  \label{fig:aniso}
\end{wrapfigure}
Consider a world model with two underlying factors: the hue $h\in S^1$ and scale $s\in[s_{\min},s_{\max}]$, with $s_{\min}>0$, of an object on a black background (\Cref{fig:dataset}). Changing the hue alters every pixel covered by the object, so its effect grows with the square root of the object's area. Changing scale alters only a thin band at the object's boundary, so its effect grows with the square
root of its perimeter. Because area and perimeter grow at different rates, so do the two factor scales:
\begin{equation}
    c_h(h,s)\propto s,
    \quad
    c_s(h,s)\propto \sqrt{s},
    \quad
    \kappa(s):=\tfrac{c_h(h,s)}{c_s(h,s)}\propto \sqrt{s}.
    \label{eq:hue-scale-rates}
\end{equation}
At fixed hue, the anisotropy $\kappa$ cannot be absorbed by a fixed choice of units: it changes by a factor of $\sqrt{s_{\max}/s_{\min}}$ across the sampled range. We verify this prediction experimentally: finite differences on the rendered images show that the two factor scales grow at different rates with scale, with power-law fits close to the analytical predictions $c_h\propto s$ and $c_s\propto\sqrt{s}$ (\Cref{fig:aniso}). \Cref{prop:aniso} in \Cref{app:hue-scale} derives these rates, gives the general scaling $c_s\propto s^\gamma$ for other rendering models, and shows that no separate reparameterization of hue and scale removes this dependence. Thus, the generator geometrically couples statistically independent factors: the image-space effect of hue depends on scale, so the factors do not induce uniform product geometry. A uniform cylinder assigns the same circumference to every scale and therefore cannot express this scale dependence, whereas a lattice whose cyclic extent grows with $s$ can. This motivates FactoMap's use of a non-uniform cyclic lattice (e.g., cone) to disentangle hue and scale.

\textbf{Factor-space topographic map.}\quad We propose FactoMap, a structured prototype learner to realize a factor-space structure as an interpretable topographic representation. It has three components: a factor-space lattice, a data-space prototype field, and a topographic learning objective.

\textit{Factor-space lattice.} For a supplied structure $\mathfrak F$, let $(\Lambda_{\mathfrak F},d_{\mathfrak F})$ be a finite lattice with an analytic distance chosen to realize that structure (see~\Cref{app:lattice-distances}). Interval and circular factors determine whether axes terminate or wrap; identifications determine which fibres meet or collapse; and the scale functions inform their relative local extent. This finite realization need not be unique.

\textit{Prototype representation.} Each site $k\in\Lambda_{\mathfrak F}$ indexes a learnable prototype $\w_k\in\R^m$, defining $W:\Lambda_{\mathfrak F}\to\R^m$, $k\mapsto\w_k$. FactoMap encodes an observation by its best-matching lattice coordinate, $q_W(\x) :=\argmin_{k\in\Lambda_{\mathfrak F}} \|\x-\w_k\|_2^2$. Thus, $q_W(\x)$ is a factor-space-indexed discrete coordinate.

\textit{Topographic learning.} For neighbourhood width $\sigma_t$, define $H_{\sigma_t}(j,k)=\exp[-d_{\mathfrak F}(j,k)^2/(2\sigma_t^2)]$. We learn the prototypes with the self-organizing-map objective~\cite{kohonen1990self}
\begin{equation}
    \mathcal L_t(W) =\mathbb E_{\x}\!\Big[ \textstyle\sum_{k\in\Lambda_{\mathfrak F}} H_{\sigma_t}\!\left(q_W(\x),k\right) \|\x-\w_k\|_2^2\Big].
    \label{eq:som-loss}
\end{equation}
The neighbourhood term encourages nearby factor-space sites to represent nearby observations. $(\Lambda_{\mathfrak F},d_{\mathfrak F},W,q_W)$ defines the learned FactoMap: the lattice encodes the supplied topology and geometry, while topographic learning realizes them in observation-space prototypes. 
%
%
\vspace{-2mm}
\section{Results}\label{sec:results}
\vspace{-4mm}
We introduce \emph{FactoShapes}, a synthetic dataset with continuously sampled hue, scale, and $x$--$y$ position unlike the discrete factors in dSprites~\cite{dsprites17}
and Shapes3D~\cite{3dshapes18} (see \Cref{app:why-factoshapes}). Here, we vary hue and scale while fixing position (\Cref{app:exp}). We compare FactoMap against $\beta$-VAE (\Cref{app:betavae}).

\begin{figure}[t]
    \centering 
    \includegraphics[width=\linewidth]{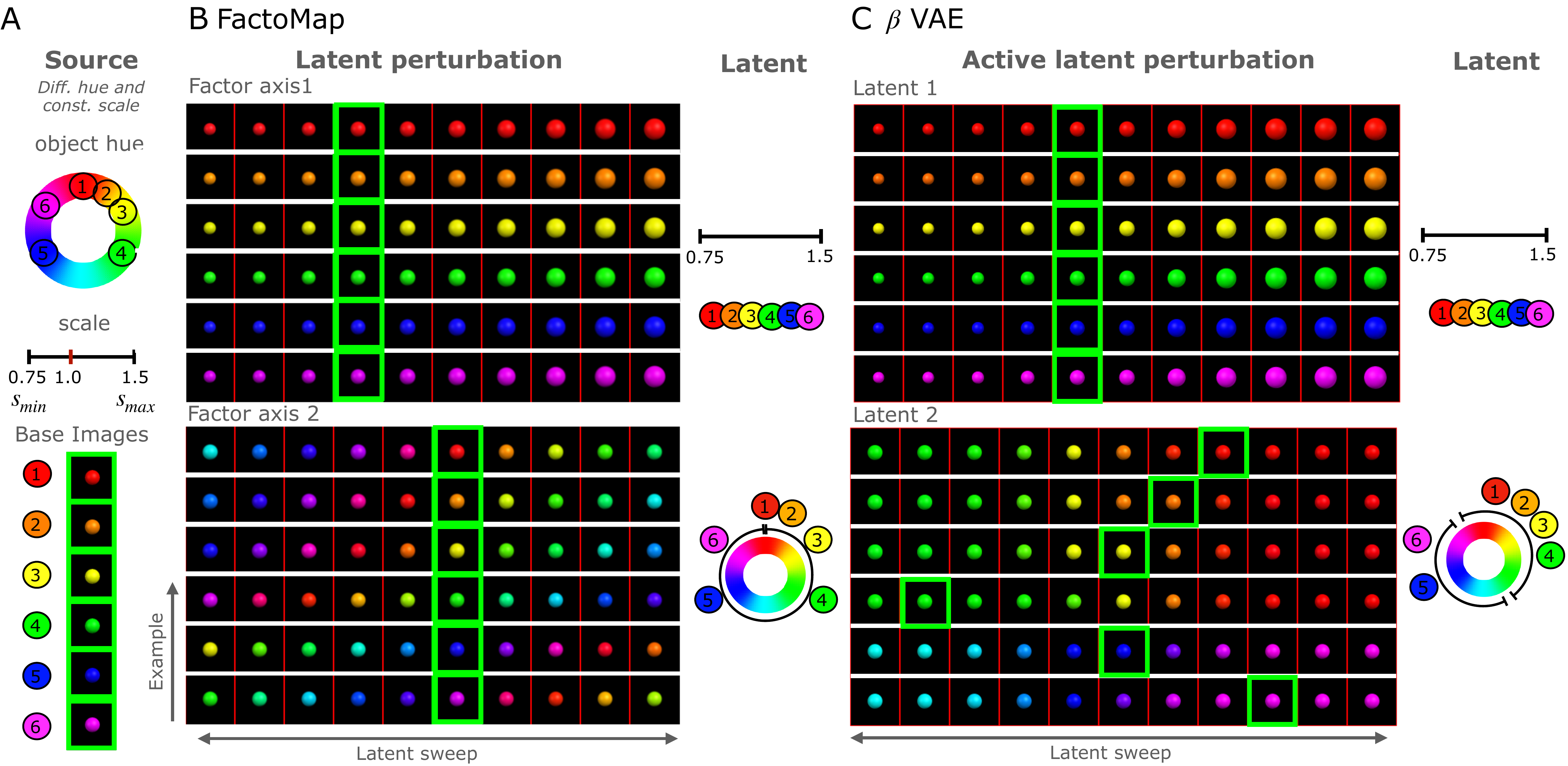}
    \vspace{-5.8mm}
    \caption{\textbf{FactoMap recovers a global cyclic hue coordinate; $\beta$-VAE does not.} \textbf{A.} Six base images have equal scale and different hues; green outlines mark the bases during each traversal. \textbf{B.} FactoMap axis 1 changes scale, while axis 2 traverses the complete hue circle from every base. \textbf{C.} One $\beta$-VAE latent changes scale, but its hue-related latent covers only a base-dependent arc of the circle.}
    \vspace{-5mm}
    \label{fig:perturb}
\end{figure}
\textbf{Matching the factor-space structure, not only its topology, is what disentangles.}\quad We compare three supplied lattices (\Cref{tab:disentanglement} {\it Left}). The grid is open in both coordinates; the cylinder wraps hue but gives it constant extent; the cone wraps hue and varies its circumference with $c_h(s)\propto s$ while spacing the scale coordinate by $c_s(s)\propto\sqrt{s}$. This qualitative organization is confirmed quantitatively in~\Cref{tab:lattice} using InfoMEC~\cite{hsu2023disentanglement}; the cone lattice substantially improves disentanglement over cylinder and grid (see~\Cref{fig:app2} for visualization of the learned prototypes).

\begin{table}
\centering
\caption{\textbf{Preserving factor-space topology and geometry improves disentanglement.} \textit{Left:} FactoMap with a cone-structured lattice outperforms cylinder and grid lattices. \textit{Right:} FactoMap better captures cyclic hue than $\beta$-VAE. NMI is reported as (hue, scale).}
\vspace{1mm}
\begin{subtable}[t]{0.48\linewidth}
\setlength{\tabcolsep}{1.0pt}
\centering
\begin{tabular}{lccc}
\toprule
    \textbf{Lattice} & \textbf{InfoM} $\uparrow$ & \textbf{InfoE} $\uparrow$ & \textbf{InfoC} $\uparrow$ \\
    \midrule
    \ding{52} Cone  & \bf 0.973 & \bf 0.842 & \bf 0.972 \\
    \ding{56} Cylinder & 0.545 & 0.698 & 0.533 \\
    \ding{56} Grid     & 0.680 & 0.755 & 0.661 \\
    \bottomrule
    \vspace{-6mm}
  \end{tabular}
  \label{tab:lattice}
\end{subtable}\hfill%
\begin{subtable}[t]{0.48\linewidth}
\setlength{\tabcolsep}{1.0pt}
\centering
\begin{tabular}{lcccc}
\toprule
    \textbf{Model} & \textbf{InfoM} $\uparrow$ & \textbf{InfoE} $\uparrow$ & \textbf{InfoC} $\uparrow$ & \textbf{NMI} \\
    \midrule
    FactoMap  & \bf 0.973 & 0.842 & \bf 0.972 & (\textbf{0.89}, 0.78) \\
    $\beta$-VAE  & 0.955 &  \bf 0.951 & 0.611 & (0.67, 0.87)\\
    \bottomrule
  \end{tabular}
  \label{tab:betavaevsfactomap}
  \end{subtable}
  \label{tab:disentanglement}
  \vspace{-5pt}
\end{table}
\textbf{High aggregate scores do not imply recovery of cyclic topology.}\quad On \emph{FactoShapes}, $\beta$-VAE achieves an InfoM score close to FactoMap and a higher InfoE score, which could suggest comparable factor recovery (\Cref{tab:disentanglement} {\it Right}). Latent traversals reveal a different picture (\Cref{fig:perturb}). Sweeping a FactoMap factor axis traverses the complete hue circle from every base image. Sweeping the $\beta$-VAE latent most associated with hue covers only a base-dependent arc. The traversals divide the hue circle into two arcs, with no traversal crossing between them, revealing a non-global cyclic coordinate.

Factor-wise NMI exposes the failure hidden by the aggregate scores. For each factor $z_i\in\{\text{hue},\text{scale}\}$, define its maximum normalized mutual information (NMI) with any latent representation as $\mathrm{NMI}^{\star}_i :=\max_j \mathrm{NMI}(z_i;\xi_j) =1-\frac{\min_j H(z_i\mid \xi_j)}{H(z_i)}$, where $\xi_j$ denotes the $j^{\text{th}}$ latent representation. $\mathrm{NMI}^{\star}_{\text{hue}}$ decreases from $0.89$ for FactoMap to $0.67$ for $\beta$-VAE, resulting in a rise of normalized residual uncertainty $H(\text{hue}\mid \xi_j)$ and showing that no single $\beta$-VAE latent captures the complete hue cycle.

\textbf{Aggregate InfoMEC scores can conceal factor-specific failures.}\quad A $\beta$-VAE trained on the full dSprites dataset achieves high aggregate scores (\Cref{tab:dsprites-sweep}). However, re-evaluating the same representation on individual factor pairs reveals substantial drops for particular pairs (\Cref{tab:pairwise_infomec}). For the orientation-scale pair, InfoC is undefined (NaN) because only one latent is identified as active, indicating that orientation is not represented by a distinct coordinate. Aggregate InfoMEC scores alone therefore do not establish that every generative factor has been recovered (see~\Cref{app:betavae}).

\begin{ack}
S.G., M.G., and B.T. would like to thank Valérie Costa for helpful discussions. S.G., M.G., and B.T. acknowledge support of Natural Sciences and Engineering Research Council of Canada (NSERC), RGPIN-2026-05959, and funding from the Canada CIFAR AI Chairs Program. Authors also thank Digital Research Alliance of Canada for computing resources. M.G.'s contributions to this work were made during co-op internships at the NeuBahar Lab.
\end{ack}

%
\bibliographystyle{unsrt}
\bibliography{references}

%
\newpage
\appendix

%
\section{Factor-space and FactoMap Details}
\label{app:factomap}
 
\textbf{Notation.}\quad We denote scalars as non-bold-lower-case $a$, vectors as bold-lower-case ${\bm a}$, and matrices as upper-case letters ${\bm A}$. $\|\cdot\|_2$ denotes the $\ell_2$ (Euclidean) norm. The generator $g:\mathcal Z\to\R^m$ maps factor configurations $\z\in \mathcal Z$ to observations $\x$. We write $\mathfrak F_g=(\mathcal Z,\approx_g,\cvec_g)$ for the factor-space structure induced by $g$ and $\mathcal F_g=\mathcal Z/\!\approx_g$ for its realized factor space. For a supplied structure $\mathfrak F$, $\Lambda_{\mathfrak F}$ denotes its finite lattice realization and $d_{\mathfrak F}$ its analytic lattice distance. Each site $k\in\Lambda_{\mathfrak F}$ indexes a prototype $\w_k\in\R^m$, and $W:\Lambda_{\mathfrak F}\to\R^m$ denotes the prototype field. The encoder $q_W(\x)$ returns the best-matching lattice site, $H_{\sigma_t}(j,k)$ is the topographic neighbourhood kernel at iteration $t$, and $\mathcal L_t(W)$ and $\widehat{\mathcal L}_t(W)$ denote the population and minibatch objectives, respectively. Finally, $p$ is the number of factors, and $K:=|\Lambda_{\mathfrak F}|$ is the number of prototypes.
 
\subsection{Formal factor-space structure}
\label{app:factor-space-details}
 
\begin{definition}[Factor-Space Structure]\label{def:facspacestruc}
Let $\mathcal Z=\prod_{i=1}^{p}Z_i$, with $Z_i\in\{I,S^1\}$ and $I=[0,1]$, and let $g:\mathcal Z\to\R^m$ be piecewise continuously differentiable. Let $\bm e_i$ denote the $i$-th coordinate direction, with increments along a circular factor taken modulo its period. The generator induces $\z\approx_g\z'
    \iff g(\z)=g(\z'),
    \
    c_i(\z)
    :=\|\frac{\partial g}{\partial z_i}(\z)\|_2,
    \
    \cvec_g:=(c_1,\ldots,c_p)$. We call $\mathfrak F_g:=(\mathcal Z,\approx_g,\cvec_g)$ the
\emph{factor-space structure} induced by $g$, and $\mathcal F_g:=\mathcal Z/\!\approx_g$, the realized factor space, equipped with the quotient topology.
\end{definition}
 
The realized factor space $\mathcal F_g=\mathcal Z/\!\approx_g$ carries the quotient topology induced by the canonical projection $\pi:\mathcal Z\to\mathcal F_g$. The scale functions are defined on $\mathcal Z$ and need not descend to the quotient: descent would require
\begin{equation}
    c_i(\z)=c_i(\z')
    \qquad\text{whenever}\qquad
    \z\approx_g\z'.
    \label{eq:scale-descent}
\end{equation}
This distinction matters at collapsed fibres, where the generator can identify an entire coordinate fibre while the remaining scale functions still vary along that fibre. Neither $\cvec_g$ nor the quotient topology uniquely determines a global lattice distance. The choice of $(\Lambda_{\mathfrak F},d_{\mathfrak F})$ is therefore a supplied finite realization of the factor-space structure and is part of the model design.
 
\subsection{Why identifications are a separate ingredient}
\label{app:identifications}
 
The second ingredient of $\mathfrak F_g=(\mathcal Z,\approx_g,\cvec_g)$ is the one our experiments do not exercise, since $s_{\min}>0$ leaves $\approx_g$ trivial and the learned prototypes lie on a conical frustum. We record here why it is nonetheless part of the structure rather than a special case of the other two.
 
\textbf{Identifications are the only route to a non-product factor space.}\quad The domain $\mathcal Z=\prod_{i=1}^{p}Z_i$ with each $Z_i\in\{I,S^1\}$ can realize only boxes, cylinders and tori. A cone with an apex, a disk, and a sphere are not products of one-dimensional factor domains, and no assignment of scale functions to a product makes one. They arise only as quotients: the cone as $(S^1\times I)/\!\approx$ with the fibre at one end identified to a point, the sphere with both ends identified.   The gallery in~\Cref{fig:factomap}A is therefore not expressible by $(\mathcal Z,\cvec_g)$ alone: the cone with an apex and the cylinder differ by an identification. The frustum (since $s_{min}>0$ ) of~\Cref{tab:lattice} has none; it shares the cylinder's topology and differs only in its scales, which is why that comparison isolates the third ingredient.
 
\textbf{Neither ingredient implies the other.}\quad The scales $\cvec_g$ and the identifications $\approx_g$ record the behaviour of $g$ at different ranges --- infinitesimal and global --- and each admits cases the other misses.
 
\emph{Closure without slowing.} Let $g$ be periodic in hue with period $\tfrac12$, so that $g(h,s)=g(h+\tfrac12,s)$ for every $s$. Then $(h,s)\approx_g(h+\tfrac12,s)$ and the realized hue factor is a circle of half the extent of $S^1$, yet $c_h$ is constant and strictly positive everywhere. No measurement of local scales, at any configuration, reveals this identification: the generator never slows along hue, it merely returns.
 
\emph{Slowing without closure.} Conversely, a generator whose hue amplitude vanishes as the object shrinks has $c_h(h,s)\to0$ as $s\to0$ (\Cref{prop:aniso}), yet for every $s>s_{\min}>0$ the hue fibre remains an honest circle on which $g$ is injective. A vanishing scale is consistent both with a fibre that collapses in the limit and with one that merely becomes slow, and the two call for different lattices: a cone that closes to a point, or a frustum that tapers and stops. Only $\approx_g$ distinguishes them.
 
\subsection{The uniform local isometry condition}
\label{app:isometry}
 
Write $\J_g(\z)=[\partial g/\partial z_1,\ldots,\partial g/\partial z_p]$ and let $\G(\z)=\J_g(\z)^\top\J_g(\z)$ be the pullback metric~\cite{kreyszig1991differential}, whose entries are
\begin{equation}
    G_{ij}(\z)
    =\left\langle
       \frac{\partial g}{\partial z_i}(\z),
       \frac{\partial g}{\partial z_j}(\z)
     \right\rangle,
    \qquad G_{ii}=c_i^2 .
    \label{eq:full-pullback-metric}
\end{equation}
Up to a fixed global scale, local isometry to uniform Euclidean factor coordinates requires
\begin{equation}
    \G(\z)=C^2\I \qquad \text{for all } \z\in\mathcal Z,
    \label{eq:isometry-condition}
\end{equation}
i.e., the diagonal scales $c_i$ are equal to one another, constant in $\z$, and the off-diagonal terms vanish. The present FactoMap construction uses the diagonal scales and analytic separable lattice distances; it does not represent nonzero off-diagonal terms. The HSV renderer used in our experiments induces such a cross-term, analyzed in~\Cref{app:exp}.
 
\subsection{The hue--scale metric}
\label{app:hue-scale}
 
\begin{proposition}[Scale-dependent anisotropy]\label{prop:aniso}
Let $g(h,s)=m_s\,\bm\rho(h)$ on a black background, with $m_s$ and $\bm\rho$ continuously differentiable and, for all $s\in[s_{\min},s_{\max}]$ and $h\in S^1$,
\begin{equation*}
\|m_s\|_2=a\,s,\qquad \|\partial_s m_s\|_2=b\,s^\gamma,\qquad \|\bm\rho(h)\|_2=q,\qquad \|\bm\rho'(h)\|_2=v,
\end{equation*}
where $a,b,q,v>0$. Then the hue and scale directions are orthogonal and
\begin{equation}
    c_h(h,s)=av\,s,
    \qquad
    c_s(h,s)=bq\,s^\gamma,
    \qquad
    \kappa(s):=\tfrac{c_h(h,s)}{c_s(h,s)}
    =\tfrac{av}{bq}\,s^{1-\gamma}.
\end{equation}
Hence $\G(h,s)=\operatorname{diag}(a^2v^2\,s^2,\;b^2q^2\,s^{2\gamma})$ and $\nicefrac{\kappa(s_{\max})}{\kappa(s_{\min})}=(\nicefrac{s_{\max}}{s_{\min}})^{1-\gamma}$. For $\gamma\neq1$, no fixed diagonal re-weighting satisfies~\eqref{eq:isometry-condition} throughout a non-degenerate scale interval.
\end{proposition}
 
\begin{proof}
Because the background is black, the generator separates into a spatial mask and a colour vector:
\begin{equation}
    \frac{\partial g}{\partial h} =m_s\,\bm\rho'(h),
    \qquad
    \frac{\partial g}{\partial s} =(\partial_s m_s)\,\bm\rho(h).
\end{equation}
Taking norms gives
\begin{equation}
    c_h=\|m_s\|_2\|\bm\rho'(h)\|_2=av\,s,
    \qquad
    c_s=\|\partial_s m_s\|_2\|\bm\rho(h)\|_2=bq\,s^\gamma.
    \label{eq:hue-scale-proof-scales}
\end{equation}
Constant colour magnitude implies $\langle\bm\rho(h),\bm\rho'(h)\rangle=\tfrac12\partial_h\|\bm\rho(h)\|_2^2=0$, so
\begin{equation}
    \left\langle\frac{\partial g}{\partial h},
                       \frac{\partial g}{\partial s}\right\rangle
    =\langle m_s,\partial_s m_s\rangle
     \langle\bm\rho'(h),\bm\rho(h)\rangle=0,
    \label{eq:hue-scale-cross-term}
\end{equation}
and the two directions are orthogonal. Substitution into $\G=\J_g^\top\J_g$ gives the stated metric and scale ratio. A fixed diagonal coordinate re-weighting multiplies $c_h$ and $c_s$ by constants and therefore cannot remove their positional dependence. In particular, the hue--hue entry remains proportional to $s^2$, so the metric cannot equal $C^2\I$ on a non-degenerate scale interval.
\end{proof}
 
\textbf{Factor-wise reparameterization.}\quad For coordinate changes $h=\eta(u)$ and $s=\sigma(t)$, the transformed scales are
\begin{equation}
    \widetilde c_u=av\,\sigma(t)|\eta'(u)|,
    \qquad
    \widetilde c_t=bq\,\sigma(t)^\gamma|\sigma'(t)|.
    \label{eq:transformed-scales}
\end{equation}
The first cannot be constant on a non-degenerate scale interval because its dependence on $t$ cannot be cancelled by $\eta'(u)$. Hence no factor-wise reparameterization makes the metric uniformly Euclidean. A scale reparameterization can equalize the two diagonal entries point-wise, but their common value remains position-dependent, yielding a conformal rather than a uniformly isometric metric.
 
\textbf{Dependence on the mask model.}\quad The exponent $\gamma$ is renderer-dependent. For a smooth self-similar mask $m_s(\bm r)=\phi(\bm r/s)$, a change of variables gives
\begin{equation}
    \|m_s\|_2\propto s,
    \qquad
    \|\partial_s m_s\|_2\propto1,
    \label{eq:self-similar-mask}
\end{equation}
so $\gamma=0$. For a smooth transition of fixed spatial width, the changing pixels occupy a band whose area grows with the object perimeter, giving
\begin{equation}
    \|\partial_s m_s\|_2^2\propto s,
    \qquad
    \|\partial_s m_s\|_2\propto\sqrt{s},
    \label{eq:fixed-width-mask}
\end{equation}
and hence $\gamma=\tfrac12$. A hard binary mask is not differentiable in $L^2$ with respect to $s$; finite-step lattice distances nevertheless remain well defined. The measured exponent $\gamma\approx0.56$ is close to the fixed-width prediction but should be interpreted as a property of the rendering pipeline. The theoretical conclusion does not depend on the precise exponent: in every smooth model above, the hue scale depends on the scale factor, so the induced geometry is not a uniform Euclidean product in the original factor coordinates. \Cref{fig:1d} shows the measured $c_s$ profile for a single varying scale factor against this prediction.
 
\textbf{Collapse at zero scale.}\quad If the model is extended by $m_0\equiv0$, then $g(h,0)=0$ for every $h\in S^1$. Thus $\approx_g$ identifies the entire hue fibre at $s=0$, while $c_h(h,s)=av\,s\to0$. Here the identification and the vanishing scale record the same collapse globally and locally; \Cref{app:identifications} gives generators for which they do not agree.
 
\subsection{Analytic lattice distances}
\label{app:lattice-distances}
FactoMap receives the finite lattice and its distance as inputs: $(\Lambda_{\mathfrak F},d_{\mathfrak F})$. No topology selection, identification estimation, or scale estimation is performed in the present work.
 
For a regular product lattice $\prod_{i=1}^{p}\{0,\ldots,S_i-1\}$, define
\begin{equation}
    \rho_i(\alpha,\beta)=
    \begin{cases}
        |\alpha-\beta|, & Z_i=I,\\[2pt]
        \min\{|\alpha-\beta|,\,S_i-|\alpha-\beta|\}, & Z_i=S^1.
    \end{cases}
\end{equation}
A weighted product distance is
\begin{equation}
    d_{\mathfrak F}(\bm j,\bm k)^2 =\sum_{i=1}^{p}\lambda_i^2\rho_i(j_i,k_i)^2,
    \label{eq:prod-lattice-dist}
\end{equation}
where $\lambda_i>0$ fixes the lattice unit of factor $i$. Open axes give lines and grids; wrapped axes give rings, cylinders, and tori. The two cases of $\rho_i$ are compared empirically on a single factor in~\Cref{app:1d}.
 
For collapsed or non-uniform structures,~\eqref{eq:prod-lattice-dist} is replaced by the analytic distance used by the corresponding lattice. These distances must specify both the collapse implementation and the dependence of one factor's extent on another. Finally, because the neighbourhood depends on $d_{\mathfrak F}/\sigma_t$, a global rescaling of the lattice distance is equivalent to rescaling the neighbourhood schedule. Distance normalization and the $\sigma_t$ schedule must therefore be specified together.
 
\subsection{Optimization}\label{app:opt}
For each mini-batch $\mathcal B_t$, we compute
\begin{equation}
    q_{W_n} =\argmin_{j\in\Lambda_{\mathfrak F}}
      \|\x_n-\w_j\|_2^2,
    \qquad \x_n\in\mathcal B_t,
    \label{eq:minibatch-bmu}
\end{equation}
and hold $q_{W_n}$ fixed during the subsequent gradient step. Conditional on these assignments, the minibatch objective is
\begin{equation}
    \widehat{\mathcal L}_t(W) =\frac{1}{|\mathcal B_t|}
      \sum_{\x_n\in\mathcal B_t}
      \sum_{k\in\Lambda_{\mathfrak F}}
      H_{\sigma_t}(q_{W_n},k)\|\x_n-\w_k\|_2^2,
    \label{eq:minibatch-loss}
\end{equation}
with prototype gradient
\begin{equation}
    \nabla_{\w_k}\widehat{\mathcal L}_t =\frac{2}{|\mathcal B_t|}
      \sum_{\x_n\in\mathcal B_t}
      H_{\sigma_t}(q_{W_n},k)(\w_k-\x_n).
    \label{eq:prototype-gradient}
\end{equation}
We update the prototypes with Adam~\cite{kingma2014adam} and recompute the best-matching sites after every update; no gradient is propagated through the $\argmin$. The experimental details specify the Adam parameters, batch size, number of updates, prototype initialization, and the complete schedule for $\sigma_t$ annealing.

%
\section{Experimental Details}\label{app:exp}
\subsection{Datasets}
\textbf{FactoShapes.}\quad We introduce \emph{FactoShapes}, a synthetic image dataset of resolution $64\times64\times3$ RGB images (\Cref{fig:dataset}). The generative factors are $v=(\text{object\_hue},\,\text{scale},\,\text{pos}_x,\,\text{pos}_y)$, so that the factor space is $S^1\times I^3$: hue is periodic (HSV hue in $[0,1)$ at fixed saturation $\langle S\rangle$ and value $\langle V\rangle$ ), while scale and the two position factors are intervals, $s\in [ s_{\min}, s_{\max}]$  and $p_x,p_y\in[-p^\star,p^\star]$. The bound $p^\star$ is obtained by a numerical solver as the largest offset for which the projected object remains fully inside the frame at the largest scale $s_{\max}$, guaranteeing that no sample is clipped and that the four factors are independently variable over the whole product range. 

\textbf{Sampling.}\quad FactoShapes supports three modes. In \emph{grid} mode the dataset is the Cartesian product of $n_h\times n_s\times n_x\times n_y$ factor values, giving $N=\langle N\rangle$ images with integer factor classes, matching the Shapes3D~\cite{3dshapes18} convention. In \emph{random-binned} mode factor values are drawn uniformly from the same discrete grid, so that each image is a random lattice point and factor classes remain defined. In \emph{continuous} mode each factor is drawn i.i.d.\ with a choice of uniform or Gaussian distribution from its continuous range; no image is repeated, and the integer class array is set to $-1$ throughout. Datasets are written in 3dshapes format: \texttt{images} of shape $(N,64,64,3)$ \texttt{uint8}, \texttt{labels} of shape $(N,4)$ \texttt{float64}, \texttt{latents\_classes} of shape $(N,4)$ \texttt{int64}, together with attributes recording the factor names and ranges. 

\Cref{fig:dataset} left shows the distribution and samples of the dataset when hue is in $[0,1)$ with S=V=1, $s\in [0.75, 1.5] \text{ and } p_x, p_y \in [-0.556, 0.556]$.

\textbf{Hue and scale as the two varying factors.}\quad We vary object\_hue (HSV hue in $[0,1)$ with $S = V = 1$) and scale ($s\in[0.75,1.5]$) while keeping the $\text{pos}_x$ and $\text{pos}_y$ fixed. We generate $100{,}000$ samples in the continuous mode with each varying factor drawn from i.i.d.\ uniformly for our experiments.~\Cref{fig:dataset} shows the sampling distribution with visuals of generated sample.
\begin{figure}
    \centering
    \includegraphics[width=1.00\linewidth]{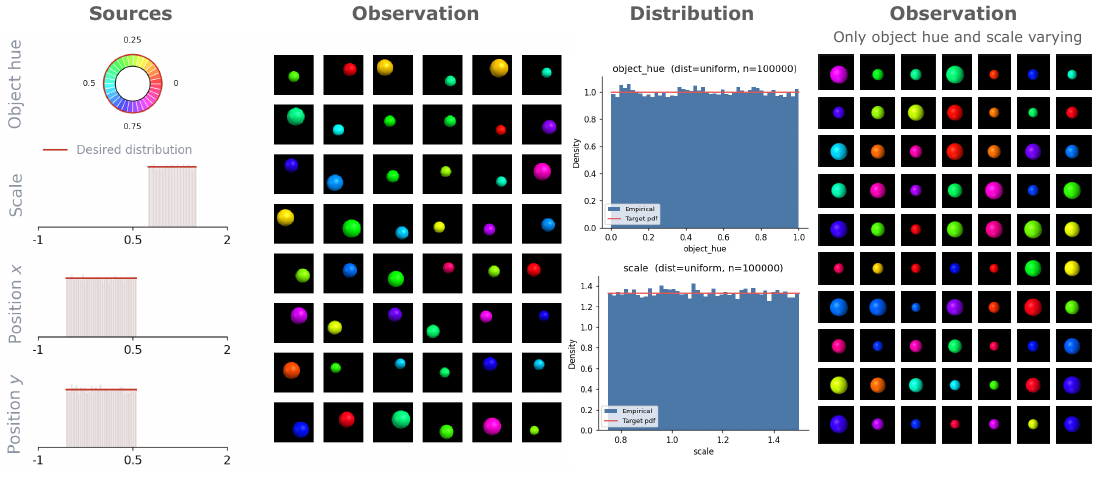}
    \caption{\textbf{FactoShapes Dataset.} Sampling distribution (uniform) and visuals of rendered sample images on the entire dataset on left and rendered sample images for only object hue and scale as varying factors (with uniform distribution) on right.}
    \label{fig:dataset}
\end{figure}

\textbf{HSV-induced metric cross-term.}\quad \Cref{prop:aniso} assumes a colour path with constant RGB norm about the black background, whereas the experiments use $\bm\rho_{\mathrm{HSV}}(h)=\operatorname{HSVtoRGB}(h,1,1)$ with normalized hue $h\in[0,1)$. This path has non-constant RGB norm and is differentiable only within each of its six linear sectors. In the first sector,
$h\in[0,\tfrac16]$, it is
\begin{equation}
    \bm\rho(h)=(1,6h,0)=(1,\tau,0),
    \qquad \tau:=6h\in[0,1].
    \label{eq:hsv-first-sector}
\end{equation}
The remaining sectors have the same form up to a channel permutation and orientation. Away from the sector boundaries,
\begin{equation}
    \|\bm\rho(h)\|_2=\sqrt{1+\tau^2},
    \qquad
    \|\bm\rho'(h)\|_2=6,
    \qquad
    |\langle\bm\rho(h),\bm\rho'(h)\rangle|=6\tau.
\end{equation}
Consequently, the normalized colour contribution to the cross-term is
\begin{equation}
    \frac{|\langle\bm\rho,\bm\rho'\rangle|}
         {\|\bm\rho\|_2\|\bm\rho'\|_2}
    =\frac{\tau}{\sqrt{1+\tau^2}}.
\end{equation}
For $g(h,s)=m_s\bm\rho(h)$, the metric cross-term factorizes as
\begin{equation}
    G_{hs}(h,s) =\langle m_s,\partial_s m_s\rangle
     \langle\bm\rho'(h),\bm\rho(h)\rangle,
\end{equation}
and is therefore generally nonzero. Under the leading-order mask model, however, the diagonal scales retain the scale dependence
\begin{equation}
    c_h(h,s)=6as\propto s,
    \qquad
    c_s(h,s)=b\sqrt{1+\tau^2}\,s^\gamma\propto s^\gamma
\end{equation}
almost everywhere. Thus, the supplied lattice matches the factor topology and the power-law dependence on scale, but not the HSV-induced cross-term or the hue-dependent prefactor of $c_s$. The results consequently demonstrate the benefit of this partial geometric match rather than general handling of non-diagonal metrics.

\textbf{dSprites.}\quad The dSprites dataset \cite{dsprites17} contains 2D shapes generated from six generative factors $v=(\text{color},\,\text{shape},\,\text{scale},\,\text{orientation}\,\text{pos}_x,\,\text{pos}_y)$. Since color is held constant across all samples, we consider the five other generative factors, resulting in $n_s=5$.

\subsection{Why we introduce FactoShapes as a new dataset?}
\label{app:why-factoshapes}
FactoShapes is designed to meet three requirements that no existing disentanglement benchmark meets together.
 
\textbf{An analyzable generator.}\quad The anisotropy result requires predicting the factor scales $c_h$ and $c_s$ analytically and then measuring them. On a black background the generator separates into a spatial mask and a colour, $g(h,s)=m_s\,\bm\rho(h)$, which is what makes the prediction of~\Cref{prop:aniso} possible. Measuring the scales requires continuous factors, so that finite differences can be taken at any location and step and a power law fitted across the whole range. Existing benchmarks provide neither. dSprites~\cite{dsprites17} has no colour factor, and its only cyclic factor, orientation, carries shape-dependent identifications (\Cref{app:dsprites}) that would confound a clean test. Shapes3D~\cite{3dshapes18} has an object-hue factor, but its 3D shading, lighting, and coloured floor and wall mean the image does not take the separable form above; and hue and scale take only 10 and 8 values on a fixed grid, which fixes both the location and the step of every finite difference.
 
\textbf{Enough data with two varying factors.}\quad Our experiments vary two factors while holding the rest fixed. In a full-factorial dataset, doing so leaves only the product of the two factors' cardinalities. In Shapes3D, object hue and scale together give $10\times8=80$ images. In dSprites, scale and orientation give $6\times40=240$, scale and shape $6\times3=18$, and the two positions $32\times32=1024$. These are too few to train a model, and too few to estimate mutual information reliably at the bin counts used by InfoMEC (\Cref{app:infomec-defs}). FactoShapes generates any number of images with only the chosen factors varying; we use $100{,}000$.
 
\textbf{Control over the factor distribution.}\quad A full-factorial dataset fixes the joint distribution of its factors: uniform over a discrete product grid. Any other distribution must be obtained by subsampling, which shrinks the subsets above further and still produces only grid values. In continuous mode, FactoShapes draws each factor independently from a chosen distribution over its continuous range  (uniform or Gaussian). We sample uniformly throughout this paper.

\subsection{Training and model parameters}
\subsubsection{FactoMap}
\textbf{Training.}\quad Training uses a single step counter,
\texttt{while step < num\_steps}. We take $\texttt{num\_steps}=60{,}000$ updates at batch size $B=128$, optimized with Adam at initial learning rate $\eta_0=0.1$ . We use the inbuilt \texttt{ExponentialLR} learning rate scheduler to reduce the $\texttt{lr\_floor= 0.001}$. The initial neighborhood radius is not set explicitly; the model initializes it to half its own topological diameter. The radius $\sigma_0$  exponentially decays to $\sigma_T=2.0$ in the training steps. The training parameters are kept same for all the three lattice : cone, cylinder and grid.

\textbf{Prototype field.}\quad For the \emph{cone} FactoMap we specify the lattice size to be $40\times20$ with prototypes $W\in {\R}^{K\times12{,}288}$ where $K = 800$ here. Knowing the fact that the rows are wrapped we have an increased number on the rows. The lattice carries \emph{cone} metric: a ring at radius $r$ has circumference $2\pi k r$ with $k=0.472$, and column $j$ sits at radius $r_j=r_0+j$ with $r_0=18.4$, so radii run from $18.4$ at the inner rim to $37.4$ at the outer rim. For \emph{cylinder and grid} FactoMap we specify the same lattice size. For cylinder the row wraps similar to cone and for grid neither row nor columns wraps. 

\subsubsection{\boldmath$\beta$-VAE}
\textbf{Training.}\quad Training uses a single step counter, \texttt{while step < num\_steps}. We take $\texttt{num\_steps}=200{,}000$ updates at batch size $B=128$, optimized with AdamW at initial learning rate $\eta_0=0.001$, and $\beta_1=0.9, \beta_2=0.99$. The reconstruction loss coefficient is fixed at $\lambda_\text{reconstruct}=1$. To investigate the effect of regularization, we evaluated seven KL divergence coefficients, $\beta_{\text{KL}}=[0.1, 0.3,1.0, 3.0, 10.0, 15.0, 20.0]$. The latent dimensionality is twice the number of generative factors for a given dataset, such that $n_z = 2n_s$. For FactoShapes, with only hue and scale as the varying factors, we get $n_z = 4$. For dSprites, we only consider five of the factors ($n_z = 10$).

\textbf{Architecture.}\quad The $\beta$-VAE architecture implements the architecture described in \cite{hsu2023disentanglement}, with a feedforward convolutional encoder, and a StyleGAN-based decoder.

\section{Additional Visualizations}
\subsection{One-dimensional factors}
\label{app:1d}
 
Before varying hue and scale jointly, we isolate each factor and vary it alone (\Cref{fig:1d}). This separates the two things a factor-space lattice must get right: the domain of a factor, and its local scale, which the two-dimensional experiments exercise together.
 
\textbf{Scale.}\quad For the interval-valued scale factor, a line-lattice FactoMap learns an ordered sequence of prototypes whose object size changes smoothly along the lattice. Finite differences on the rendered images show that the local scale speed increases approximately as $c_s(s)\propto\sqrt{s}$, consistent with the fixed-width rendering model of~\Cref{app:hue-scale}, and the cumulative prototype distance along the lattice is correspondingly nonlinear.
 
\textbf{Hue.}\quad Applying the same open lattice to the periodic hue factor introduces a cut: hue values near $0$ and $1$ are adjacent in factor space, but their best-matching sites lie at opposite ends of the lattice ($\Delta_{\mathrm{lattice}}=39$ of a possible $39$). A ring lattice restores their adjacency ($\Delta_{\mathrm{lattice}}=1$) and produces a continuous cycle of hue prototypes. Local continuity therefore does not determine global closure; periodicity has to be supplied in the lattice. The prototype-distance profiles show the two structures being realized: along the ring, distance from a reference prototype grows to the antipode and falls back, as a cyclic factor requires.
 
At fixed scale, the measured $c_h$ is close to the constant predicted by~\Cref{prop:aniso}, with shallow dips at the HSV sector boundaries where the colour path $\bm\rho_{\mathrm{HSV}}$ is non-differentiable and its norm is largest; this is the hue-dependent prefactor discussed in~\Cref{app:exp}.
 
\begin{figure}[ht]
    \centering
    \includegraphics[width=0.85\linewidth]{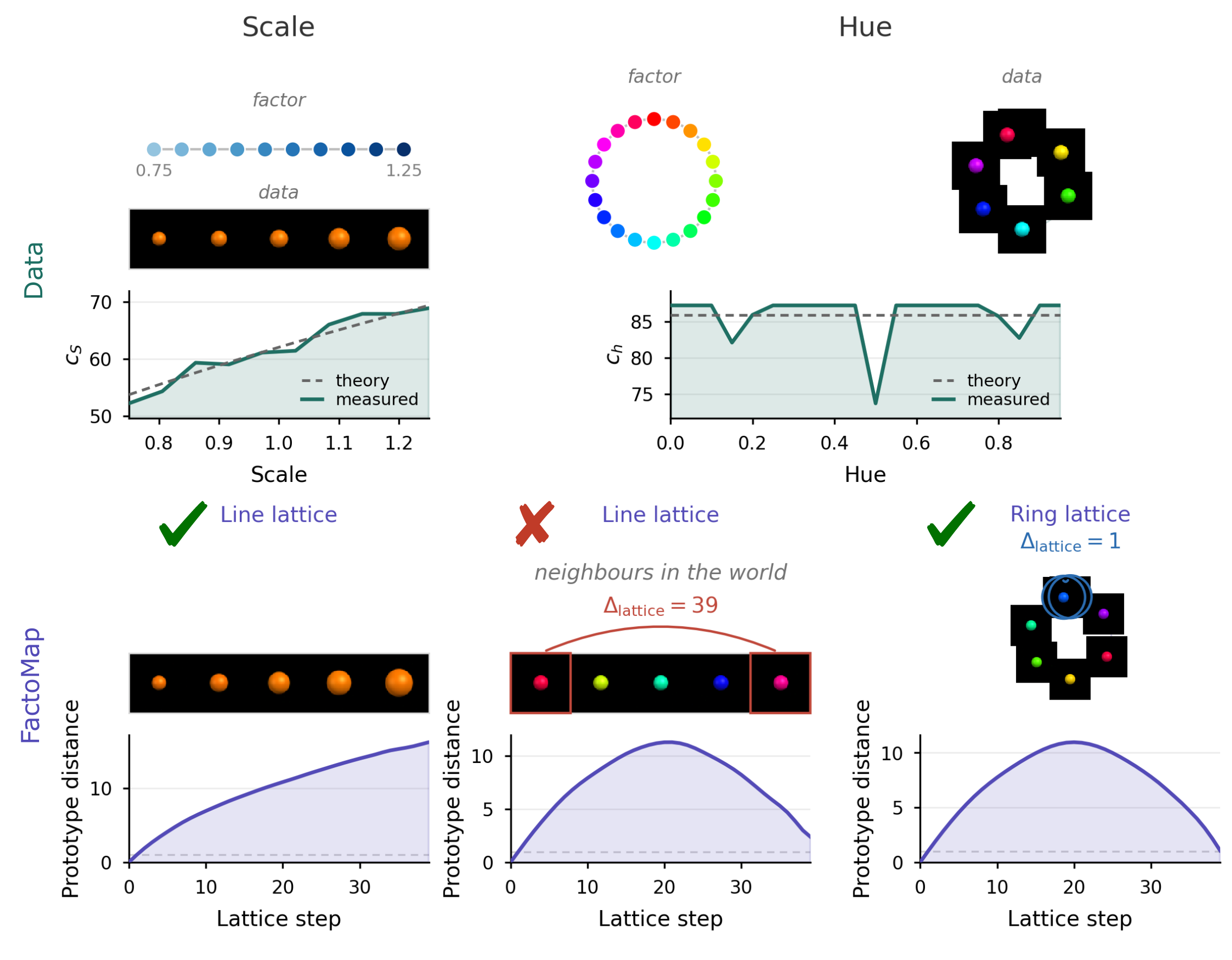}
    \caption{\textbf{One-dimensional factors: matching the domain and the local scale.}
    \emph{Top (Data).} The scale factor is an interval and the hue factor a circle. Measured image-space speeds (solid) against the analytical predictions (dashed): $c_s$ grows with scale, while $c_h$ is near-constant across hue, dipping at the HSV sector boundaries.
    \emph{Bottom (FactoMap).} A line lattice represents scale continuously, with a nonlinear cumulative prototype distance reflecting the position-dependent $c_s$. For periodic hue, an open line lattice places the adjacent endpoint hues at opposite ends ($\Delta_{\mathrm{lattice}}=39$), whereas a ring lattice restores adjacency ($\Delta_{\mathrm{lattice}}=1$) and yields a closed cycle of prototypes.}
    \label{fig:1d}
\end{figure}
 
\subsection{Prototype learning  with two varying factors}
We learn grid, cylinder and cone prototype lattices on FactoShapes (\Cref{fig:dataset}), where hue and scale are the only varying factors. On the grid, hue and scale do not run systematically along the lattice axes, so neither factor is recoverable from any axis direction. The failure is informative: the learned prototypes still wrap in hue, indicating that the map is trying to close a cycle the grid cannot support. Matching this cyclic structure requires a lattice with a circular direction, hence we try a cylinder geometry to accommodate the cyclic hue factor. Interestingly, this attempt did not disentangle well. This strengthens to need to consider the growth of the hue factor with scale and the taper of a cone accommodates this. On the cone lattice, hue aligns with the circular direction and scale with the axial direction (\Cref{fig:app2}), and \Cref{tab:disentanglement} confirms this with InfoMEC metric scores. 

\begin{figure}[ht]
    \centering
    \includegraphics[width=1\linewidth]{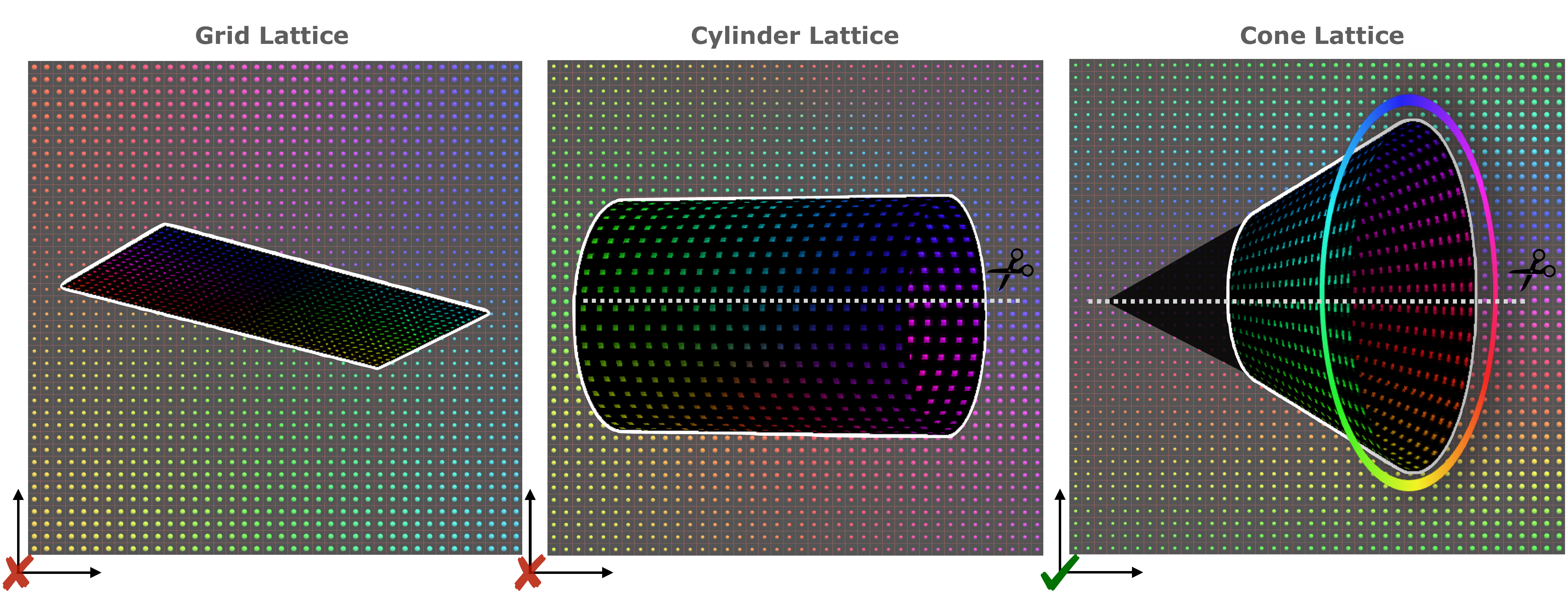}
    \caption{\textbf{Visualizations of learned prototypes by FactoMap n a grid, cylinder and cone lattice.} The cylinder and cone is cut to unfold the lattice to visualization the alignment of prototypes.}
    \label{fig:app2}
\end{figure}

\section{Evaluating Disentanglement with InfoMEC}
\label{app:infomec}

\subsection{Definitions}
\label{app:infomec-defs}

Let $z_1,\ldots,z_p$ be the factors and $\xi_1,\ldots,\xi_{n_a}$ the active coordinates of the learned representation: the latents of a $\beta$-VAE, or the two lattice coordinates of a FactoMap. The selection criterion for a latent to be active is same as~\cite{hsu2023disentanglement}; inactive latents are discarded before any score is computed. Factors and representation coordinates are discretized into a fixed number of bins (we report 10, 20, 50, and an automatic choice) and the normalized mutual information
\begin{equation}
    \mathrm{NMI}_{ij}
    = \mathrm{NMI}(z_i;\xi_j)
    =\frac{I(z_i;\xi_j)}{H(z_i)}
    =1-\frac{H(z_i\mid \xi_j)}{H(z_i)}\in[0,1]
    \label{eq:nmi}
\end{equation}
forms a $p\times n_a$ matrix. $\mathrm{NMI}_{ij}$ is the fraction of the uncertainty about factor $i$ that coordinate $j$ removes.

\textbf{InfoM and InfoC.}\quad Modularity and compactness summarize how information is distributed across this matrix:
\begin{align}
    \mathrm{InfoM}
    &=\left(\frac{1}{n_a}\sum_{j=1}^{n_a}\frac{\max_i\mathrm{NMI}_{ij}}{\sum_{i}\mathrm{NMI}_{ij}}-\frac{1}{p}\right)\bigg/\left(1-\frac{1}{p}\right),
    \label{eq:infom}\\
    \mathrm{InfoC}
    &=\left(\frac{1}{p}\sum_{i=1}^{p}\frac{\max_j\mathrm{NMI}_{ij}}{\sum_{j}\mathrm{NMI}_{ij}}-\frac{1}{n_a}\right)\bigg/\left(1-\frac{1}{n_a}\right).
    \label{eq:infoc}
\end{align}
InfoM is high when each coordinate carries information about a single factor; InfoC is high when each factor's information is concentrated in a single coordinate. Both depend only on the \emph{shape} of the NMI matrix: InfoM is unchanged if any column is rescaled by a positive constant, and InfoC if any row is. Neither, therefore, measures \emph{how much} of a factor a coordinate captures.

\textbf{InfoE.}\quad Explicitness measures how well each factor can be predicted from all active coordinates by a linear model, averaged over factors. 

\textbf{Per-factor NMI.}\quad We additionally report, for each factor, the information carried by its single most informative coordinate,
\begin{equation}
    \mathrm{NMI}^\star_i
    =\max_j\mathrm{NMI}_{ij}
    =1-\frac{\min_j H(z_i\mid \xi_j)}{H(z_i)} .
    \label{eq:nmi-star}
\end{equation}
Unlike InfoM and InfoC it measures magnitude: a low $\mathrm{NMI}^\star_i$ means that no single coordinate determines factor $i$ entirely, whatever the remaining coordinates do. It is the quantity reported as NMI (hue, scale) in~\Cref{tab:disentanglement}.

\section{Is the \boldmath$\beta$-VAE Disentangling Every Factor?}
\label{app:betavae}

Architecture and training follow~\Cref{app:exp}. For each dataset we sweep $\beta_{\mathrm{KL}}$ over the values 0.1, 0.3, 1, 3, 10, 15 and 20, and select the $\beta$ with the best InfoMEC metric , and rerun the selected value for four further seeds (we report the mean as standard deviation was negligible).

\subsection{FactoShapes: hue and scale}
\label{app:betavae-hs}

\FloatBarrier

\begin{table}[ht]
    \centering
    \begin{tabular}{cccccc}
        \hline
        n bins & $\beta$ & InfoM & InfoE (Regression) & InfoE (Classification) & InfoC \\
        \hline
        10     & 0.1 & 0.568 & 0.787 & 0.703 & 0.382 \\
               & \textbf{0.3} & \textbf{0.983} & \textbf{0.788} & \textbf{0.951} & \textbf{0.625} \\
               & 1   & 0.985 & 0.768 & 0.941 & 0.620 \\
               & 3   & 0.986 & 0.803 & 0.927 & 0.616 \\
               & 10  & 0.981 & 0.771 & 0.933 & 0.615 \\
               & 15  & 0.982 & 0.747 & 0.925 & 0.625 \\
               & 20  & 0.293 & 0.519 & 0.543 & 0.033 \\
        \hline
        20     & 0.1 & 0.589 & 0.787 & 0.703 & 0.390 \\
               & \textbf{0.3} & \textbf{0.969} & \textbf{0.788} & \textbf{0.951} & \textbf{0.614} \\
               & 1   & 0.977 & 0.768 & 0.941 & 0.617 \\
               & 3   & 0.975 & 0.803 & 0.927 & 0.612 \\
               & 10  & 0.965 & 0.771 & 0.933 & 0.602 \\  
               & 15  & 0.959 & 0.747 & 0.925 & 0.602 \\
               & 20  & 0.236 & 0.519 & 0.543 & 0.024 \\
        \hline
        50     & 0.1 & 0.578 & 0.787 & 0.703 & 0.379 \\
               & \textbf{0.3} & \textbf{0.936} & \textbf{0.788} & \textbf{0.951} & \textbf{0.590} \\
               & 1   & 0.951 & 0.768 & 0.941 & 0.600 \\
               & 3   & 0.946 & 0.803 & 0.927 & 0.588 \\
               & 10  & 0.933 & 0.771 & 0.933 & 0.582 \\
               & 15  & 0.920 & 0.747 & 0.925 & 0.573 \\
               & 20  & 0.179 & 0.519 & 0.543 & 0.027 \\
        \hline
        auto   & 0.1 & 0.590 & 0.787 & 0.703 & 0.381 \\
               & \textbf{0.3} & \textbf{0.955} & \textbf{0.788} & \textbf{0.951} & \textbf{0.611} \\
               & 1   & 0.964 & 0.768 & 0.941 & 0.602 \\
               & 3   & 0.963 & 0.803 & 0.927 & 0.601 \\
               & 10  & 0.954 & 0.771 & 0.933 & 0.594 \\
               & 15  & 0.945 & 0.747 & 0.925 & 0.591 \\
               & 20  & 0.183 & 0.519 & 0.543 & 0.026 \\
    \end{tabular}
    \caption{\textbf{Dataset: FactoShapes, Model: $\bm{\beta}$-VAE.} Aggregate InfoMEC Scores. All $\beta$ values were first evaluated, after which four additional runs were performed only for the best-performing $\beta$ to compute the reported mean. These runs exhibited negligible variance, so standard deviations are omitted for clarity.}
    \label{tab:hs-sweep}
\end{table}

\textbf{The NMI matrix.}\quad \Cref{fig:nmi-hs} shows where the information sits. Both models concentrate scale in a single coordinate. FactoMap concentrates hue in its angular lattice coordinate ($\mathrm{NMI}=0.89$); the $\beta$-VAE splits hue across two latents ($0.67$ and $0.58$), neither of which determines it.

\begin{figure}[ht]
    \centering
    \includegraphics[width=\textwidth]{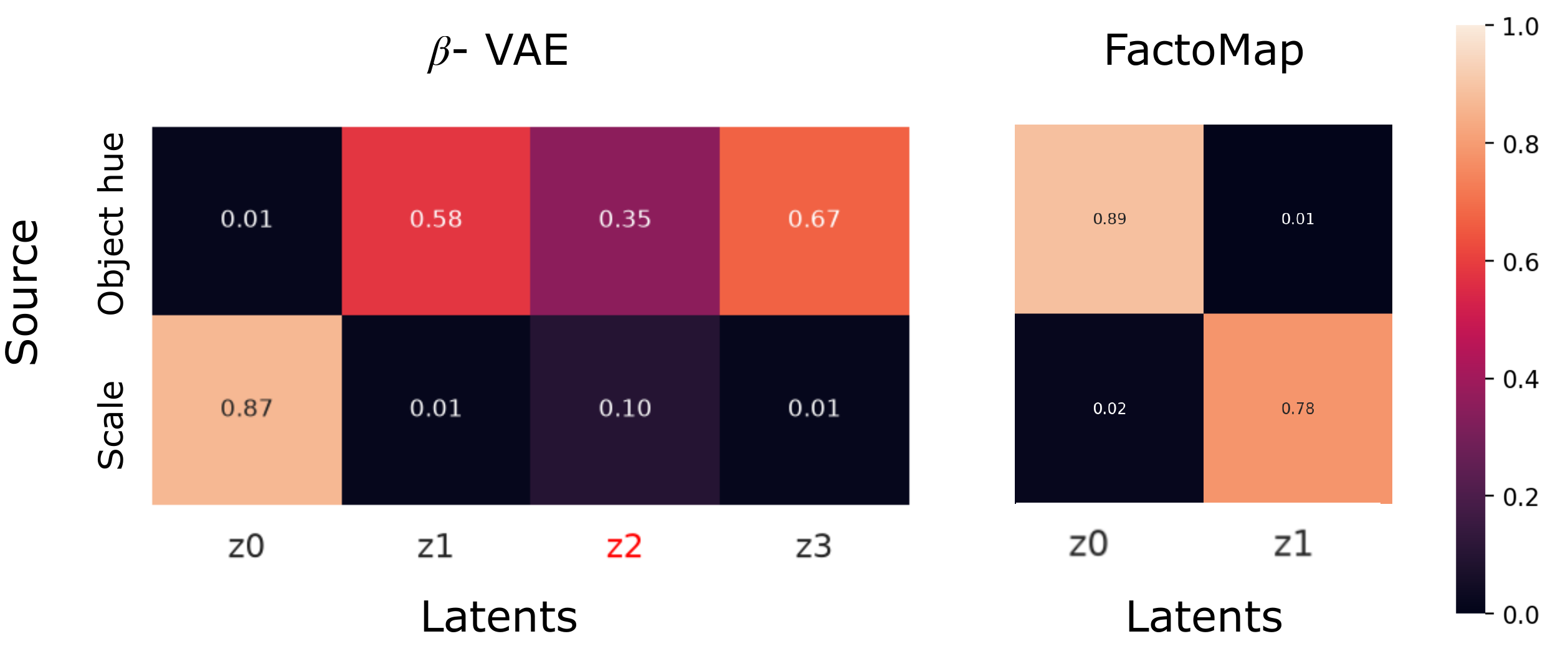}
    \caption{\textbf{NMI matrices on FactoShapes.} Rows are factors, columns are representation coordinates; entries are $\mathrm{NMI}_{ij}$ on a common $[0,1]$ scale. Both models concentrate scale in one coordinate. FactoMap concentrates hue in its angular coordinate, whereas the $\beta$-VAE splits it across two latents; inactive latents are red.}
    \label{fig:nmi-hs}
\end{figure}

\textbf{The perturbations.}\quad Traversing the $\beta$-VAE's most hue-informative latent (\Cref{fig:betavae-latperb-scaleconst}, \Cref{fig:betavae-latperb-hueconst}) runs from green--cyan at one end to red--magenta at the other for every base image; the base image decides only whether the sweep passes through yellow and orange or through blue and purple. This shows that the latent takes its extreme values at two roughly antipodal hues, each intermediate value is reached on both sides of the circle. Two further observations are consistent with this reading. Looking at \Cref{tab:hs-sweep} the InfoC plateau across $\beta$ suggests a structural rather than a tuning failure (no choice of $\beta$ lets a single real latent carry a circle). And regression InfoE stays well below classification InfoE ($0.788$ vs.\ $0.951$ at $\beta=0.3$): a linear classifier can carve angular sectors out of two latents, whereas linear regression onto $h$ cannot undo the wrap. Latent perturbations for FactoMap (~\Cref{fig:factomap-latperb-scaleconst},~\Cref{fig:factomap-latperb-hueconst}) doesn't have the coverage problem, it captures the full hue. 

\begin{figure}[ht]
    \centering
    \includegraphics[width=\textwidth, height=0.75\textheight, keepaspectratio]{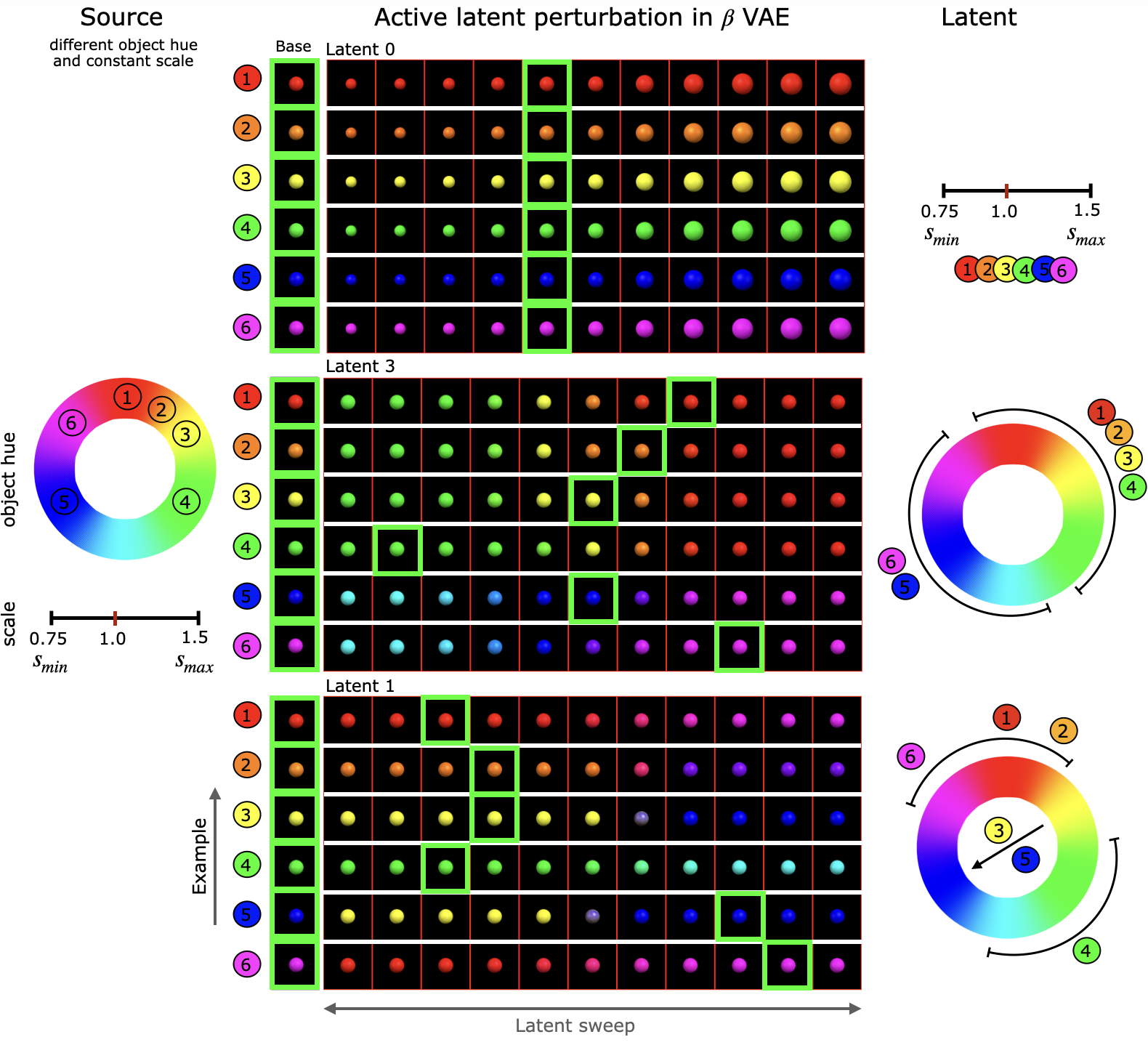}
    \caption{\textbf{The $\bm\beta$-VAE fails to represent the complete hue cycle with a single continuous latent coordinate.} With scale fixed, its latent traversals on FactoShapes partition hue into base-dependent arcs and exhibit discontinuous jumps, including abrupt yellow-to-blue transitions. Each row starts from a base image of a different hue, marked by a green outline, and columns show values of the traversed latent. Latents are ordered by decreasing maximum factor-wise NMI.}
    \label{fig:betavae-latperb-scaleconst}
\end{figure}

\begin{figure}[ht]
    \centering
    \includegraphics[width=\textwidth, height=0.75\textheight, keepaspectratio]{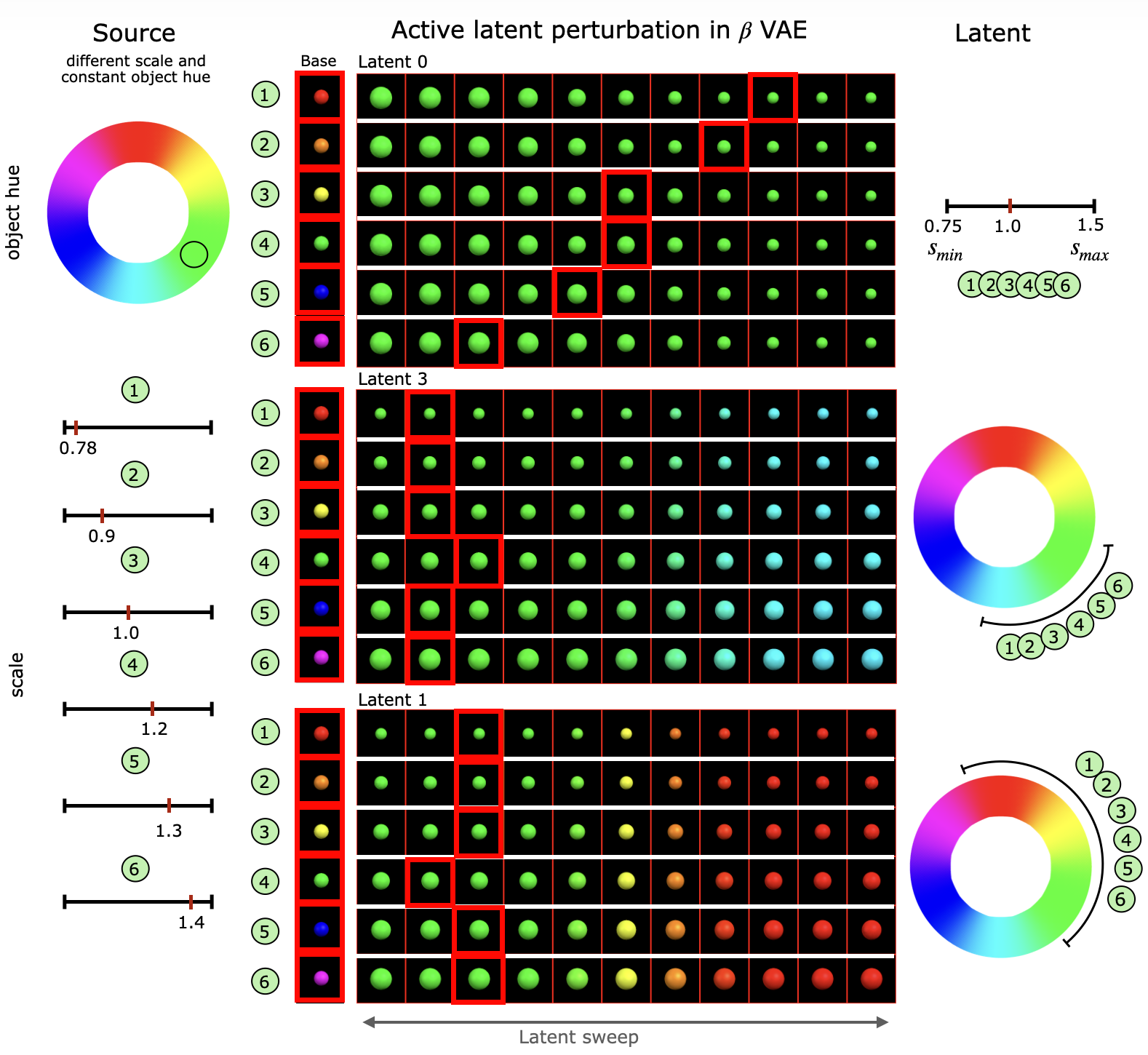}
    \caption{\textbf{$\bm\beta$-VAE partitions the cyclic hue factor across latents.} The base hue is fixed at green, while rows correspond to base images of different scales, highlighted by red boxes. Traversing latent 3 covers only the green-to-light-blue arc, whereas latent 1 covers approximately half of the hue circle; neither traverses the complete cycle. Columns show values of the traversed latent. Latents are ordered by decreasing maximum factor-wise NMI.}
    \label{fig:betavae-latperb-hueconst}
\end{figure}

\begin{figure}[ht]
    \centering
    \includegraphics[width=\textwidth, height=0.75\textheight, keepaspectratio]{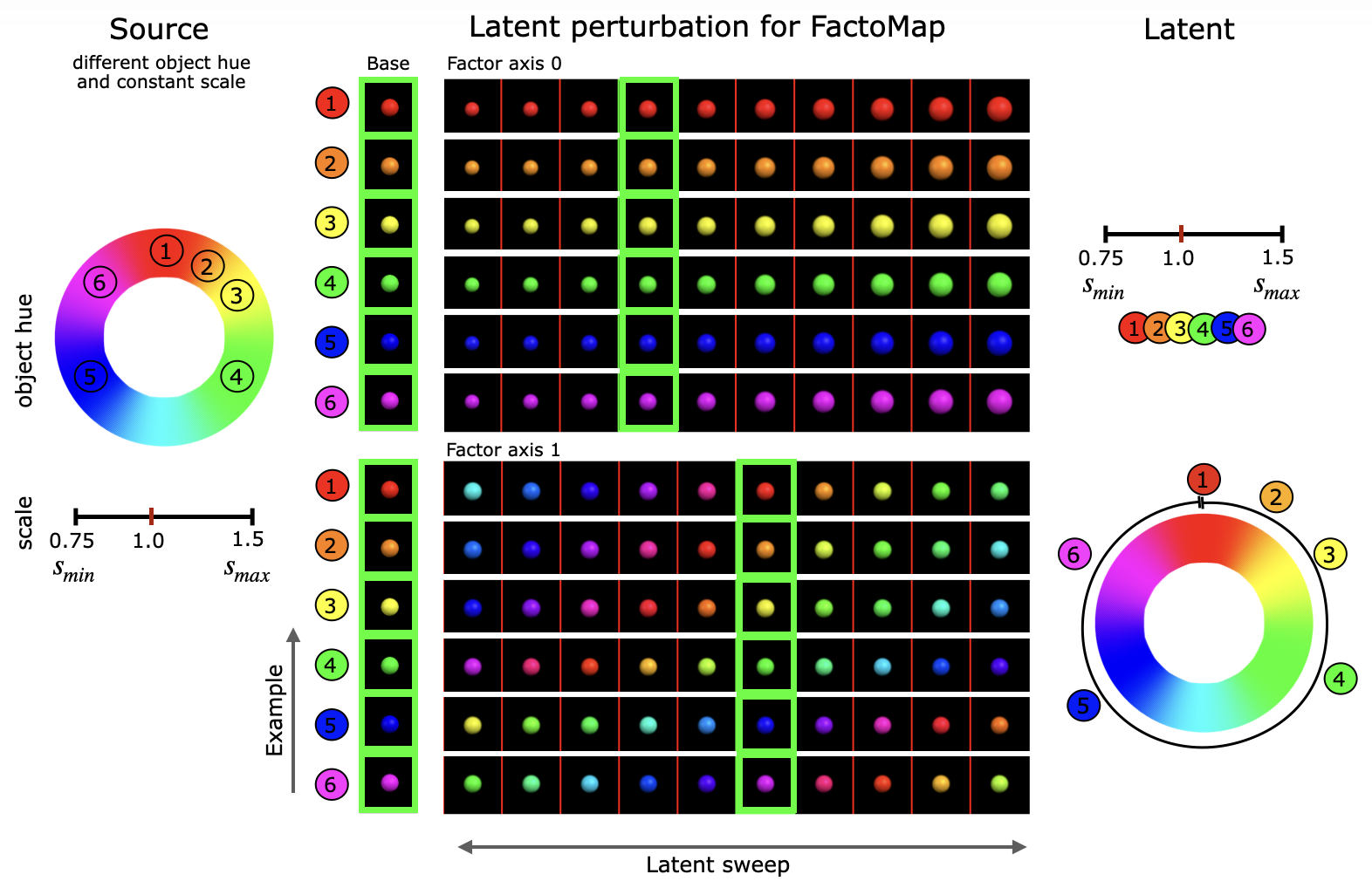}
    \caption{\textbf{FactoMap represents the complete hue cycle within a single factor axis.} Traversals on FactoShapes are shown with scale fixed. Each row starts from a base image of a different hue, marked by a green outline, and columns follow values along the traversed axis. The learned axes separate hue and scale, while the hue axis covers the complete cycle and preserves its periodic closure.}
    \label{fig:factomap-latperb-scaleconst}
\end{figure}

\begin{figure}[ht]
    \centering
    \includegraphics[width=\textwidth, height=0.75\textheight, keepaspectratio]{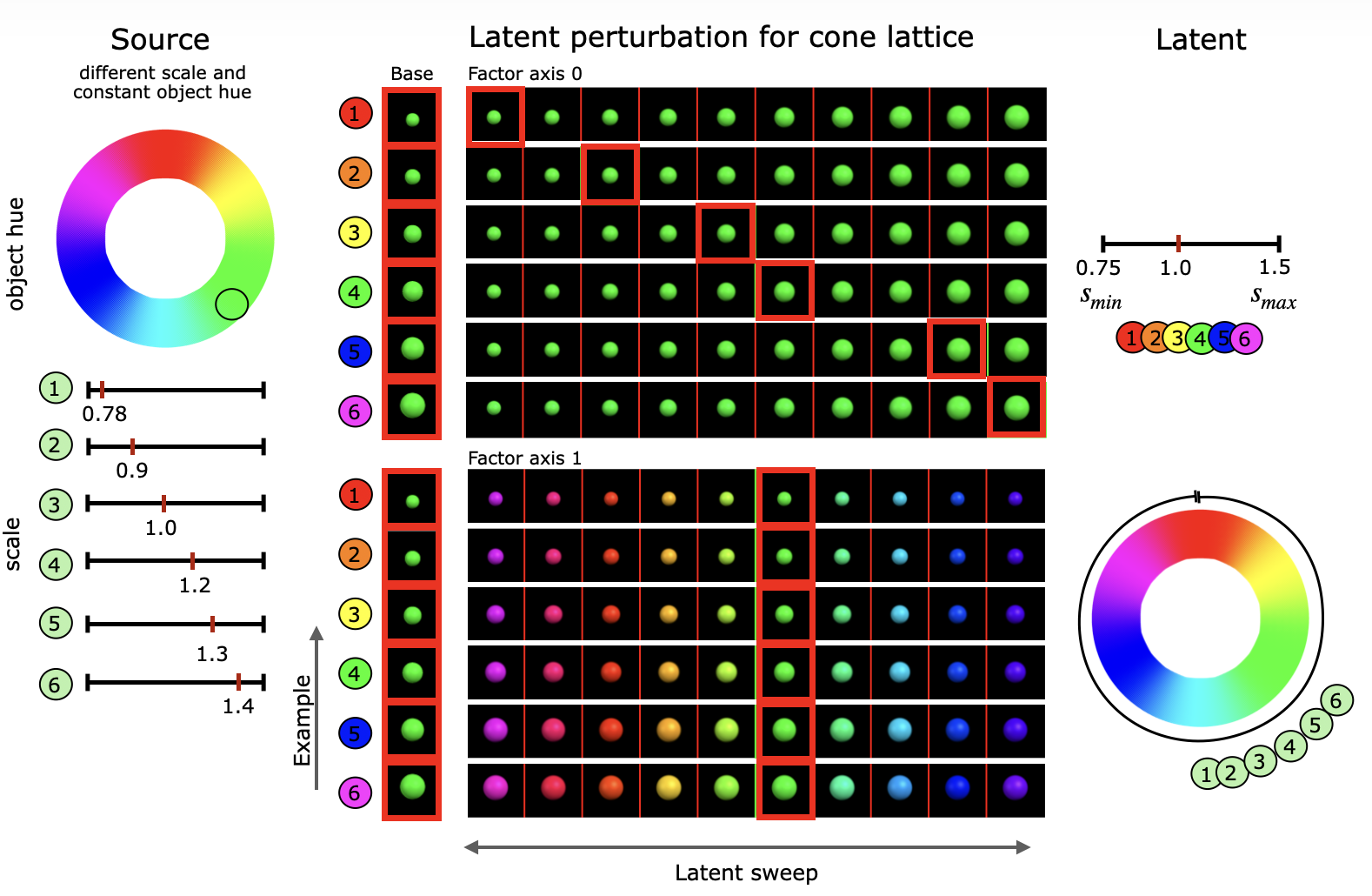}
    \caption{\textbf{FactoMap disentangles continuous scale from cyclic hue.} Factor-axis traversals begin from base images of different scales, marked by red outlines. Factor axis~0 changes object scale without altering hue, whereas factor axis~1 traverses the complete hue cycle at every base scale without altering scale. Columns show positions along the traversed axis.}
    \label{fig:factomap-latperb-hueconst}
\end{figure}

\FloatBarrier

\subsection{dSprites}
\label{app:dsprites}

dSprites~\cite{dsprites17} has five varying factors: shape, scale, orientation, and horizontal and vertical position. We train on all five and evaluate the selected model ($\beta=10$) twice: on the full dataset, and on subsets in which only two factors vary and the remaining three are held fixed. Holding three factors fixed leaves few images per subset (240 for scale--orientation, 18 for scale--shape; \Cref{app:why-factoshapes}), so the pairwise NMI estimates are biased upward: latents inactive on the full dataset reach NMI up to 0.96 with scale in \Cref{fig:nmi-dsprites-pairwise}, against at most 0.20 in \Cref{fig:nmi-dsprites-all}. We read \Cref{tab:pairwise_infomec} for its qualitative pattern only. All $\beta$ values were first evaluated, after which four additional runs were performed only for the best-performing $\beta$ to compute the reported mean. These runs exhibited negligible variance, so standard deviations are omitted.

\FloatBarrier

\begin{table}[ht]
    \centering
    \begin{tabular}{cccccc}
        \hline
        n bins & $\beta$ & InfoM & InfoE (Regression) & InfoE (Classification) & InfoC \\
        \hline
        10     & 0.1 & 0.592 & 0.560 & 0.419 & 0.505 \\
               & 0.3 & 0.607 & 0.576 & 0.442 & 0.687 \\
               & 1   & 0.364 & 0.567 & 0.405 & 0.394 \\
               & 3   & 0.820 & 0.587 & 0.420 & 0.827 \\
               & \textbf{10}  & \textbf{0.823} & \textbf{0.591} & \textbf{0.385} & \textbf{0.877} \\
               & 15  & 0.839 & 0.597 & 0.391 & 0.785 \\
               & 20  & 0.847 & 0.597 & 0.378 & 0.795 \\
        \hline
        20     & 0.1 & 0.555 & 0.560 & 0.419 & 0.480 \\
               & 0.3 & 0.566 & 0.576 & 0.442 & 0.653 \\
               & 1   & 0.317 & 0.567 & 0.405 & 0.451 \\
               & 3   & 0.779 & 0.587 & 0.420 & 0.832 \\
               & \textbf{10}  & \textbf{0.784} & \textbf{0.591} & \textbf{0.385} & \textbf{0.871} \\
               & 15  & 0.774 & 0.597 & 0.391 & 0.763 \\
               & 20  & 0.780 & 0.597 & 0.378 & 0.770 \\
        \hline
        50     & 0.1 & 0.488 & 0.560 & 0.419 & 0.425 \\
               & 0.3 & 0.501 & 0.576 & 0.442 & 0.584 \\
               & 1   & 0.273 & 0.567 & 0.405 & 0.466 \\
               & 3   & 0.705 & 0.587 & 0.420 & 0.825 \\
               & \textbf{10}  & \textbf{0.709} & \textbf{0.591} & \textbf{0.385} & \textbf{0.824} \\
               & 15  & 0.701 & 0.597 & 0.391 & 0.730 \\
               & 20  & 0.713 & 0.597 & 0.378 & 0.731 \\
        \hline
        auto   & 0.1 & 0.493 & 0.560 & 0.419 & 0.429 \\
               & 0.3 & 0.520 & 0.576 & 0.442 & 0.603 \\
               & 1   & 0.276 & 0.567 & 0.405 & 0.458 \\
               & 3   & 0.689 & 0.587 & 0.420 & 0.822 \\
               & \textbf{10}  & \textbf{0.709} & \textbf{0.591} & \textbf{0.385} & \textbf{0.817} \\
               & 15  & 0.712 & 0.597 & 0.391 & 0.735 \\
               & 20  & 0.726 & 0.597 & 0.378 & 0.737 \\
    \end{tabular}
    \caption{\textbf{Dataset: dsprites, Model: $\bm{\beta}$-VAE.} Aggregate InfoMEC Scores.}
    \label{tab:dsprites-sweep}
\end{table}

\textbf{All factors.}\quad \Cref{fig:nmi-dsprites-all} shows the corresponding NMI matrix for evaluation on all factors varying.

\begin{figure}[ht]
    \centering
    \includegraphics[width=\textwidth, height=0.45\textheight, keepaspectratio]{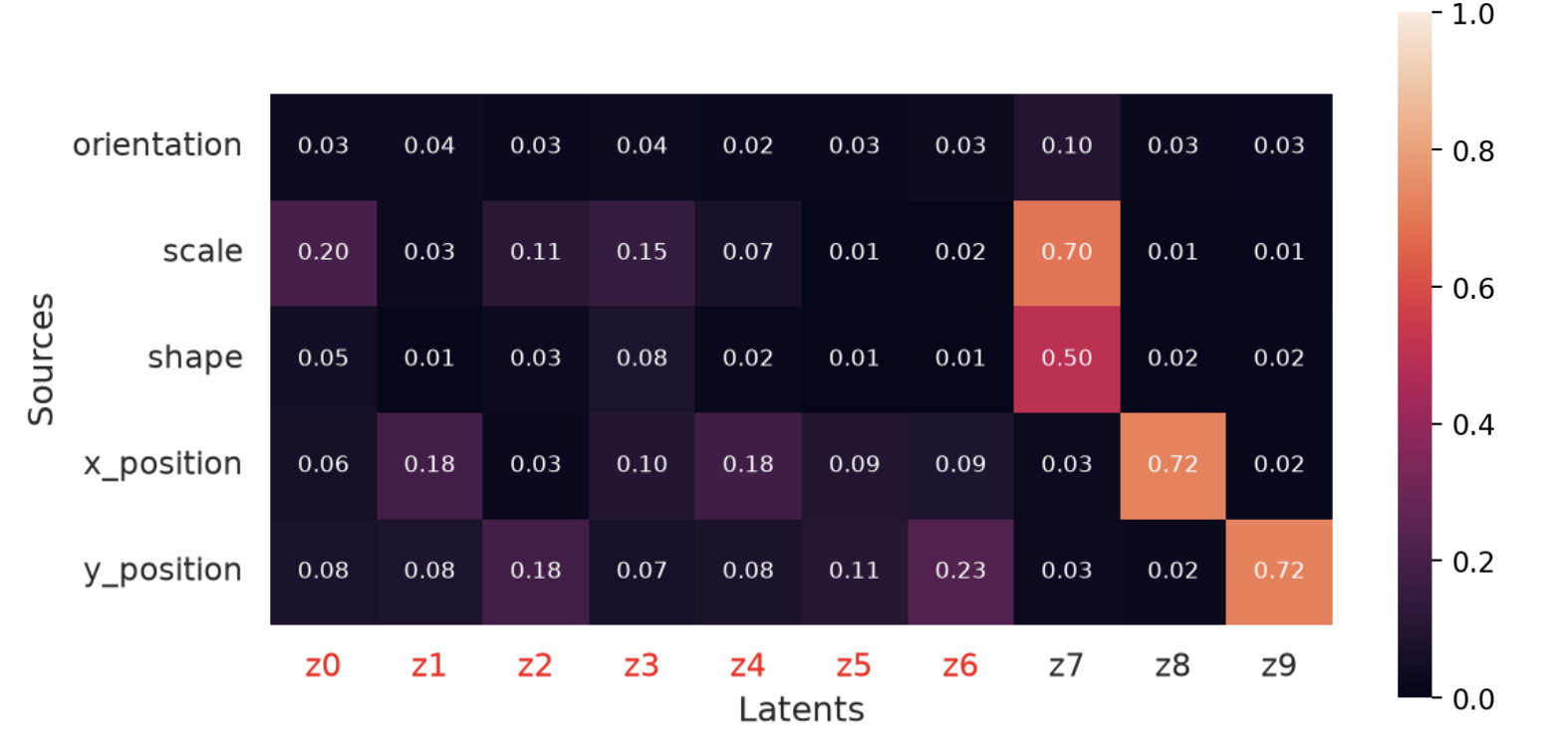}
    \caption{\textbf{NMI matrix of the $\bm\beta$-VAE on dSprites, all factors} ($\beta=10$). Rows are the five factors, columns the ten latents; inactive latents are red.}
    \label{fig:nmi-dsprites-all}
\end{figure}

\begin{table}[ht]
    \centering
    \begin{tabular}{cccccc}
        \hline
        Varying Factors & n bins & InfoM & InfoE (Regression) & InfoE (Classification) & InfoC \\
        \hline
        Position X--Position Y
            & 10   & 0.666 & 0.996 & 0.895 & 0.660 \\
        & 20   & 0.657 & 0.996 & 0.895 & 0.658 \\
        & 50   & 0.598 & 0.996 & 0.895 & 0.597 \\
        & auto & 0.666 & 0.996 & 0.895 & 0.673 \\
        \hline
        Scale--Shape
            & 10   & 0.241 & 0.990 & 0.999   & 0.241 \\
        & 20   & 0.047 & 0.990 & 0.999 & 0.047 \\
        & 50   & 0.017 & 0.990 & 0.999 & 0.059 \\
        & auto & 0.291 & 0.990 & 0.999 & 0.285 \\
        \hline
        Scale--Orientation
            & 10   & 0.836 & 0.506 & 0.629 & NaN \\
        & 20   & 0.684 & 0.506 & 0.629 & NaN \\
        & 50   & 0.501 & 0.506 & 0.629 & NaN \\
        & auto & 0.792 & 0.506 & 0.629 & NaN \\
        \hline
        Orientation--Shape
            & 10   & 0.473 & 0.625 & 0.691 & 0.150 \\
        & 20   & 0.342 & 0.625 & 0.691 & 0.109 \\
        & 50   & 0.203 & 0.625 & 0.691 & 0.085 \\
        & auto & 0.453 & 0.625 & 0.691 & 0.129 \\
        \hline
    \end{tabular}
    \caption{\textbf{Dataset: dsprites, Model: $\bm{\beta}$-VAE, $\bm{\beta}$: 10.} Pairwise Factor Evaluation InfoMEC Scores.}
    \label{tab:pairwise_infomec}
\end{table}

\textbf{Pairs of factors.}\quad Restricted to two varying factors, the same model scores lower on InfoM for every pair except for factor pair Position X and Position Y and the degenerate scale--orientation case discussed below. 

On \emph{scale--orientation} a single latent is active, and it encodes scale (See latent 7 in NMI matrix~\Cref{fig:nmi-dsprites-pairwise}); no latent varies with orientation. When given only a single active latent, $\text{InfoC}$ is mathematically undefined ($\text{NaN}$), as the normalization factor $1-1/n_a$ evaluates to zero. More fundamentally, in the case of a single latent, $\text{InfoMEC}$ metrics become trivial and uninformative, as disentanglement inherently requires multiple dimensions across which information can be partitioned. For $\text{InfoC}$, which measures the extent to which a single generative factor is being captured by a single latent variable, such a degenerate case cannot be meaningfully evaluted.

\emph{A cyclic factor with identifications.} Looking at the pairwise NMI matrix~\Cref{fig:nmi-dsprites-pairwise} and then \Cref{tab:pairwise_infomec} the two pairs involving orientation have the lowest InfoE of the four, in both variants: orientation is the factor the representation recovers worst. Orientation in dSprites is cyclic with a shape-dependent period: an ellipse rotated by $\pi$ and a square rotated by $\nicefrac{\pi}{2}$ render identically. In the language of~\Cref{def:facspacestruc}, the orientation--shape factor space is not the product $S^1\times\{\text{square},\text{ellipse},\text{heart}\}$, since each shape's orientation circle is identified under a different symmetry (the closure without slowing of~\Cref{app:identifications}). We do not test this causally, but it is the structure a Euclidean latent cannot represent.

\textbf{Aggregate score misleads.}\quad InfoM and InfoC average the shape of the NMI matrix over all factors and latents. Well-separated factors, such as the two positions, raise the average and offset entanglement or loss elsewhere. Evaluating on pairs of factors localizes the failure that the aggregate absorbs.

\begin{figure}[ht]
    \centering
    \includegraphics[width=\textwidth, keepaspectratio]{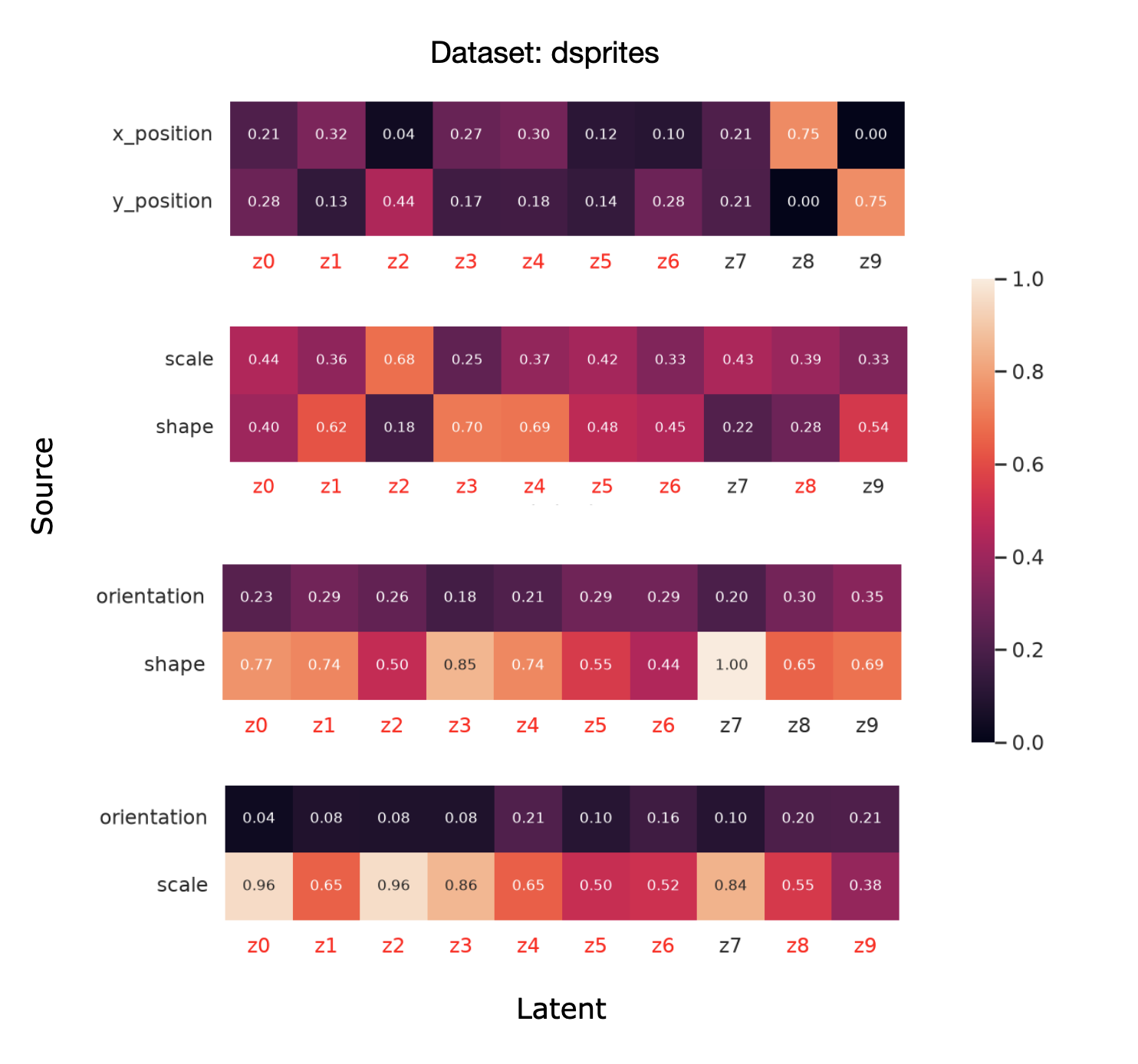}
    \caption{\textbf{Pairwise NMI matrices of the $\bm\beta$-VAE on dSprites} ($\beta=10$). One panel per pair of varying factors; columns keep the same latent index in every panel, so that a latent can be followed across pairs. Inactive latents are red.}
    \label{fig:nmi-dsprites-pairwise}
\end{figure}

\FloatBarrier

\end{document}